\documentclass{article}

\usepackage{arxiv}
\makeatletter
\AtBeginDocument{%
  \pagestyle{fancy}%
  \rhead{}%
}

\renewcommand{\keywords}[1]{\par\noindent\textbf{Keywords:} #1}
\makeatother
\usepackage{booktabs}
\usepackage{tabularx}
\usepackage[utf8]{inputenc} % allow utf-8 input
\usepackage[T1]{fontenc}    % use 8-bit T1 fonts
\usepackage{amssymb} % for \varnothing
\usepackage{amsthm} % for proof environment

\usepackage{url}            % simple URL typesetting
\usepackage{amsfonts}       % blackboard math symbols
\usepackage{microtype}      % microtypography
\usepackage{graphicx}
\usepackage{amsmath}
\usepackage{array} 
\usepackage{natbib}
\usepackage{lineno}         % for \linenumbers
\usepackage{longtable}
\usepackage{pdflscape}

\usepackage{enumitem}
\setlist[itemize]{leftmargin=1.6em}
\setlist[enumerate]{leftmargin=1.8em}

\usepackage{calc}

\usepackage{authblk}

\makeatletter
\renewcommand\AB@affilnote[1]{\textsuperscript{#1}\,}
\makeatother

\usepackage[colorlinks=true,linkcolor=black,citecolor=blue,urlcolor=blue]{hyperref}

\newtheorem{definition}{Definition}

\newtheorem{proposition}{Proposition}

\newtheorem{corollary}{Corollary}

\newcommand{\Sys}{\mathsf{Sys}}

\newcommand{\Hom}{\mathrm{Hom}}

\newcommand{\diam}{\diamondsuit}
\newcommand{\supp}{\text{supp}}

\usepackage{tikz}
\usetikzlibrary{positioning}
\usepackage{tikz-3dplot}
\usetikzlibrary{arrows.meta,positioning,calc,decorations.markings}
\usepackage{pgfplots}
\pgfplotsset{compat=1.18}
\definecolor{diagramblue}{RGB}{13,101,139}
\definecolor{stepIgray}{RGB}{242,243,245}
\definecolor{stepIIteal}{RGB}{226,242,239}

\definecolor{stepIteal}{RGB}{220,240,237}
\definecolor{stepIIamber}{RGB}{250,235,200}

\newcommand{\stepboxvisual}[2]{%
  \tikz[baseline=(X.base)]{
    \node[
      fill=#1,
      rounded corners=2pt,
      line width=0.00pt,
      inner xsep=5.5pt,
      inner ysep=1.6pt,
      minimum height=1ex,
      align=center
    ] (X) {%
      \strut
      \makebox[\widthof{\textbf{Step~II}}][c]{\textbf{#2}}%
    };
  }%
}

\DeclareRobustCommand{\StepI}{%
  \texorpdfstring{\stepboxvisual{stepIteal}{Step~I}}{Step I}%
}

\DeclareRobustCommand{\StepII}{%
  \texorpdfstring{\stepboxvisual{stepIIamber}{Step~II}}{Step II}%
}
\title{Convex AI Compositionality and the Governance of AI System Populations}

\author[1,2,3]{Andrea Ferrario%
\thanks{\href{mailto:aferrario@ethz.ch}{\texttt{aferrario@ethz.ch}}}%
\hspace{0.35em}}
\affil[1]{Institute of Biomedical Ethics and History of Medicine, University of Z\"urich, Z\"urich, Switzerland}
\affil[2]{SUPSI, Dalle Molle Institute for Artificial Intelligence (IDSIA), Lugano, Switzerland}
\affil[3]{ETH Z\"urich, Z\"urich, Switzerland}

\begin{document}

\maketitle

\begin{abstract}
AI governance increasingly requires providers and public authorities to reason about multiple AI instantiations, alternative versions, and deployment configurations of  multiple AI systems. Yet current regulation remains predominantly single-system-centric, acknowledging such multiplicity only sparsely without treating collections of related AI systems as governance objects.
This creates an \emph{AI population governance problem}: determining which instantiations can be meaningfully considered together and how their changing configurations can be represented and monitored. The first requirement has recently been addressed through trustworthiness-based accounts of AI identity. We address the second by introducing \emph{convex AI compositionality}: a formal representation of the configurations generated by finite AI system populations that uses convex spaces. The core idea is that convex compositions of the operational states that a population of AI system instantiations may occupy over time are compatible with lifecycle reachability across the population and can preserve the formal identity relations between these systems. Well-known statistical and geometric constructions, such as weighted state distributions and convex hulls, become AI governance tools for distinguishing operational states, population weights, heterogeneity, and AI configuration change across different governance modes while remaining compatible, under stated conditions, with lifecycle reachability and AI identity. We illustrate our AI population governance framework through distributed healthcare deployments and controlled deployment of recruitment AI variants.
\end{abstract}


%%%%%%%%%%%%%%%%%%%%%%%%%%%%%%%%%%%%%%%%%%%%%%%%%%%%%
%%%%%%%%%%%%%%%%%%%%%%%%%%%%%%%%%%%%%%%%%%%%%%%%%%%%%
\section{Introduction}
\label{section:introduction}
The mass production and commodification of artificial intelligence (AI) systems has made AI governance increasingly concerned with populations of related AI system instantiations distributed across space and time. On the one hand, providers may operate several copies of one system across different institutions, maintain production and canary versions of their online service simultaneously, or allocate inference traffic among alternative configurations---including 10,000 concurrent agents to address the Navier--Stokes Millennium Prize Problem\footnote{\url{https://openai.com/index/navier-stokes-solution/}}---while users personalize AI system instances with their preferences, prompt styles, and shared data. On the other hand, governance actors must determine which deployments can be compared under common and known evidence, which changes may remain covered by an existing assessment, whether a problem  local in space or time indicates a population-wide issue, and which other instantiations in the population should be included when corrective action is required. Yet current regulation remains predominantly single-system-centric. For instance, the European AI Act (AIA) structures lifecycle obligations around \emph{the} high-risk AI system, while only sparsely acknowledging multiplicity through references to \emph{all} AI systems concerned and, in Regulation (EU) 2026/1744, to ``individual units of the same type and model'' \citep{EUAIAct2024,EUAIActOmnibus2026}. Such provisions expose multiplicity without making collections of related instantiations a governance object. We call these multiple comparable AI system instantiations \emph{AI populations}. Their differences may also determine how risks are distributed across affected individuals:
% degradation may remain concentrated in particular contexts or populations, while traffic allocation may expose different shares of users to different variants. 
determining the relevant population and its evolution therefore affects  accountability and corrective action.

In this work, \textbf{we address this AI population governance gap through a two-step approach}. First, we argue that governance actors must determine which concrete AI system instantiations belong to a common comparison domain and how they remain identifiable across deployment and change
(\emph{comparability and identity}). Second, once such a finite population has been identified, AI governance requires an abstract representation of the states occupied by its members and of the population configurations generated by those states (\emph{state compositionality and governance}).

The first step has been recently developed through trustworthiness-based AI identity criteria, which specify when systems belong to a common AI system \emph{type} and when instantiations remain identical across deployments and over time \citep{Ferrario2025TrustworthinessMetaphysics}. Their category-theoretic extension represents \emph{quantified AI states}, admissible lifecycle transformations, and reachability among states \citep{Ferrario2026CategoryIdentity}, providing the structure needed to identify AI instantiations and their trajectories over time.
Here, we develop the second and novel step. Using convex spaces, equivalently barycentric algebras \citep{fritz2009convex,zamojska2024barycentric,Stone1949Barycentric}, we represent AI population configurations through convex weighted compositions of quantified AI states. In real-world scenarios, these weights may represent fractions of (to be) deployed instantiations, processed cases, or inference traffic. Then, convex compositionality represents spaces of \emph{possible composed profiles} generated by alternative allocations over AI population states, allowing us to study their compatibility with the identity relations established in the first step. We embed convex compositionality into an AI population governance workflow that can cover retrospective, operational, and prospective AI governance. Common constructions in statistical mechanics and convex geometry, i.e., state probability distributions, their barycenters, variances, and convex hulls \citep{berlinsky2019statistical,hug2020lectures}, provide AI population governance objects for tracking operational prevalence, heterogeneity,  and configuration change across deployments and time. Finally, we illustrate our approach through distributed healthcare deployments and controlled deployment of recruitment AI variants, and relate these constructions back to the AI population governance gap. The paper makes three contributions:
\begin{enumerate}[leftmargin=1.4em]
    \item \textbf{We define convex AI compositionality} using basic convex space mathematics \citep{fritz2009convex,zamojska2024barycentric,Stone1949Barycentric} and establish conditions under which AI state convex composition preserves trustworthiness levels and AI identity relations \citep{Ferrario2025TrustworthinessMetaphysics,Ferrario2026CategoryIdentity}.
    \item \textbf{We represent finite AI system populations} through weighted state distributions and convex hulls over anchor states, distinguishing their realized and possible configurations, aggregate conditions, and heterogeneity.
    \item \textbf{We introduce an AI population governance workflow} using these objects as tools to study AI state and population changes, and weight redistributions in retrospective, operational, and prospective governance modes.
\end{enumerate}

%%%%%%%%%%%%%%%%%%%%%%%%%%%%%%%%%%%%%%%%%%%%%%%%%%%%%%%%
%%%%%%%%%%%%%%%%%%%%%%%%%%%%%%%%%%%%%%%%%%%%%%%%%%%%%%%%
\section{Motivation: Two Illustrative Cases of AI State Composition}
\label{section:motivation}
We start by presenting two illustrative cases that help isolate the problem. The first concerns the monitoring of multiple deployed \emph{AI instantiations}, namely, concrete AI system realizations over time;\footnote{
\emph{AI instantiation} denotes an AI realization that we select as a unit of our AI governance analysis. We prefer this terminology to others such as  \emph{copies}, \emph{tokens} and cognates. An AI instantiation need not correspond one-to-one to a model artifact, process, virtual machine, container, endpoint, tenant, or user
account. Several AI instantiations may share technical infrastructure, while
one deployment may itself comprise several models or runtime replicas. Here, following the EU AI Act \citep{EUAIAct2024}, an AI system is understood as a machine-based system operating with varying degrees of autonomy, potentially adapting after deployment, and inferring from its inputs how to generate outputs---including predictions, content, recommendations, or decisions---that may affect physical or virtual environments (Art.~3, EU AI Act).} the second concerns the selection and governance of alternative AI system designs before and during deployment.

%%%%%%%%%%%%%%%%%%%%%%%%%%%%%%%%%%%%%%%%%%%%%%%%%%%%%%%%%%%%%
\subsection{Monitoring Multiple AI System Instantiations Across Healthcare Institutions} 
\label{subsection:AI_healthcare_monitoring}
Consider an AI system developed by a medical technology company to support the detection of neurological lesions in magnetic resonance images. The system is successfully deployed in 15 healthcare institutions across the country. All instantiations share the same intended purpose and specifications of correct functioning. The provider therefore treats these deployments as instantiations of one \emph{AI system type} and relies on a common body of evidence to manage their maintenance. These systems do not, however, operate under identical conditions. Different hospitals use scanners produced by different manufacturers, apply different image processing protocols, and serve different populations. Some institutions integrate the system directly into their radiology workflow, while others require clinicians to upload examinations manually. These modifications gradually produce measurable differences among the instantiations.
After some time, a certain AI system starts exhibiting a growing number of false-positive lesion detections after the hospital in which it is deployed changes one of its imaging protocols. Two smaller institutions show no comparable decline in accuracy, but their clinicians override the system considerably more often than clinicians in all the other locations. As a result, the provider now faces a collection of AI system instantiations that, at the moment of deployment, were arguably identical given the common intended use and  operational standards, but over time, seem to occupy different \emph{states} determined by their measured performance, robustness, human-oversight, and operational-efficiency values.

Managing each AI instantiation separately is not sufficient. Managers from the participating hospitals begin discussing the findings jointly and ask the provider to demonstrate that the system remains trustworthy across the healthcare network. 
They need to determine whether the common validation evidence remains applicable to every deployment and whether some instantiations require local corrective action.
Regulators may also ask whether an incident observed in one hospital is an isolated problem or evidence of broader degradation affecting the collection of AI system instantiations. The immediate tasks for the provider are to define criteria for comparing the AI instantiations and characterize the configuration of operational states that this AI system collection \emph{may} occupy under specified operational conditions. These would help answer the questions: which variations remain compatible with appropriate functioning within the system type, and when does an instantiation transition to a different state? How should the \emph{admissible configurations of this population be represented}, and how can their evolution over time be monitored for AI governance purposes?

%%%%%%%%%%%%%%%%%%%%%%%%%%%%%%%%%%%%%%%%%%%%%%%%%%%%%%%%%%%%%
\subsection{Managing LLM-Based Recruitment Variants Through Controlled Deployment}
\label{subsection:Alt_AI_systems}
Consider a provider operating an AI system that assists a large employer in comparing candidate applications with job requirements and producing recommendations for  recruiters. The provider already operates a current production version but is developing a new candidate version, referred to as the \emph{canary}, together with ten  variants. These variants differ in their underlying models, retrieval sources, explanation mechanisms, and thresholds for human review. However, they serve the same intended use and are evaluated through a common protocol. This protocol is described through operational dimensions that include recommendation quality, robustness, human oversight, and privacy. The variants may achieve the same overall degree of correct functioning through different combinations of these properties.

Such a system is subject to the requirements applicable to high-risk AI systems under the EU AI Act \citep{EUAIAct2024}. Thus, the provider adopts a controlled deployment strategy. Incoming applications constitute the system's \emph{inference traffic}: the stream of cases submitted to the system for processing. The provider initially assigns applications across the active variants, with the current version processing 75\%, the canary 10\%, and selected alternatives the remaining 15\%. Then, the first task is to describe \emph{all possible configurations}  that can be supported during testing. Here, a configuration specifies which variants are active and the proportion of inference traffic assigned to each. The provider then selects an initial configuration and allows it to evolve as monitoring evidence accumulates. Over time, the canary may receive a larger traffic share, one alternative may be withdrawn, and another may enter active testing after completing shadow evaluation. The complete deployment cannot be assessed solely by examining each variant separately. In fact, even when every active variant is individually acceptable, the functioning of the service depends on how applications are distributed among them. The provider must therefore assess both the state of each variant and the overall profile generated by their changing proportions. This entails answering the key questions: which combinations of variants are \emph{admissible}? 
% Can individually acceptable variants produce an unacceptable deployment when used together? 
How should a change in one variant or its traffic share affect the assessment of the complete service? Finally, how should the selected configuration and its evolution be represented, monitored, and compared over time?

%%%%%%%%%%%%%%%%%%%%%%%%%%%%%%%%%%%%%%%%%%%%%%%%%%%%%%%%%%%%%%
\subsection{The Problem: From AI Populations to Population Configurations...and Their Governance}
These two cases exemplify the problem of governing multiple AI system instantiations across deployments and over time. In the healthcare case, the relevant objects are instantiations of a \emph{common AI type}, while in the recruitment case, they are available \emph{variants} of an AI-based service. We call a finite collection of AI system instantiations that can be meaningfully identified and compared across deployments and over time an \emph{AI system population}. The examples also show that the members of an AI population occupy operational \emph{states}, namely, representations of system properties such as performance, robustness, safety, and transparency that are relevant for AI governance in the context where these instantiations are designed and deployed.\footnote{At this stage, we use the term \emph{state} in this general sense; Section~\ref{section:category_account} provides the formal representation.}
Here, several instantiations may occupy the same state, while the state occupied by an individual instantiation may change with operating environment or measured properties. An AI population can therefore realize different \emph{configurations of states} across deployments and over time. 
The configuration observed at any one moment is only one among those relevant to its lifecycle governance. 

In summary, our examples suggest that the AI population governance problem requires (1) formalizing AI populations, and (2) representing the state configurations these populations can possibly occupy across deployments and over time. Then, such representations should allow governance actors to monitor the evolution of these AI populations and trigger appropriate intervention. Following these intuitions, we address the problem in a two-step sequential approach:

\begin{itemize}
    \item \textbf{\StepI~(comparability and identity)}: establish criteria for determining which AI system instantiations can be meaningfully compared across deployments and over time as members of an AI population;

    \item \textbf{\StepII~(state compositionality and governance)}: construct a representation of the state configurations possibly occupied by AI populations and use it for monitoring and to support AI governance interventions.
\end{itemize}

% Note that the two steps are sequential.
% : a population  and its configurations are meaningful only after the domain of comparability of their instantiations has been fixed. 
Before developing them, however, we examine how regulation addresses this AI population governance problem.

%%%%%%%%%%%%%%%%%%%%%%%%%%%%%%%%%%%%%%%%%%%%%%%%%%%%%%
%%%%%%%%%%%%%%%%%%%%%%%%%%%%%%%%%%%%%%%%%%%%%%%%%%%%%%
\section{The Problem and AI Regulation}
\label{section:regulatory_motivation}
Current AI regulation remains predominantly organized around identifiable individual governance objects: \emph{the} AI system, \emph{the} service, \emph{the} unit, or \emph{the} device. 
% This is understandable as regulatory obligations need an object to which documentation, monitoring, modification, and corrective action can attach. However, regulation cannot entirely avoid relying on some account of what the system is, when it persists through change, and which concrete realizations belong together. This is one reason why metaphysics-inspired terminology and governance meet in this work.
The AIA exemplifies this single-system-centric approach: under Article 11 and Annex IV technical documentation tracks versions and changes of a high-risk system, Article~17 governs its quality management, Article~43 addresses its substantial modifications and conformity assessment, and Article~72 requires post-market monitoring throughout the system lifetime \citep{EUAIAct2024}. Yet multiplicity appears at the boundaries of the regulation. Article~79(4), for instance, may require corrective action with respect to ``all the AI systems concerned'' that an operator has made available on the Union market. The regulation does not, however, provide a general account of how these systems should be related, compared, or collected. This limitation is also visible in regulation (EU) 2026/1744 (``Digital Omnibus on AI'') \citep{EUAIActOmnibus2026}. 
Recital 39, explaining the amendment to the transitional regime under Article 111, distinguishes an AI system  from ``individual units of the same type and model.'' The provision keeps several units under the same regime while the high-risk AI system design remains unchanged,  acknowledging multiplicity within the same \emph{type} and \emph{model}, but without treating the units as a governance object. (Neither regulation defines \emph{type} or \emph{model}.)

Other regulatory regimes address similar issues. Under the Digital Services Act, systemic-risk assessment concerns a regulated ``service and its related systems, including algorithmic systems'' (Article~34), and Article~35 contemplates testing and adapting those systems \citep{EU_DSA2022}. Here, multiplicity is largely absorbed into a higher-level governance object---\emph{the service}---rather than represented explicitly as a collection of related algorithmic systems. A short-lived legislative precedent appeared in Colorado's SB24-205 \citep{ColoradoSB24205}, enacted in 2024 and repealed and reenacted with a different automated-decision-making regime by SB26-189 in 2026. SB24-205 allowed one impact assessment to address a ``comparable set'' of high-risk AI systems deployed by a deployer, while requiring review of individual deployed systems \citep{ColoradoSB24205}.

It is important to note that engineering and machine learning research \emph{does} address forms of system multiplicity. Model management systems track and compare model versions \citep{vartak2016modeldb,schelter2018challenges,hao2024mgit}, while MLOps organizes their development, deployment, and monitoring  \citep{sculley2015hidden,kreuzberger2023machine}. Controlled-experimentation and canary infrastructures operate concurrent variants and allocate production traffic among them \citep{tang2010overlapping,lindon2022rapid}, while physical AI research treats multiple deployed robots as a collective operational unit (called a \emph{fleet}) \citep{hoque2023fleet,wang2024robot}. 
These approaches govern AI system multiplicity for different operations, but do not provide a common lifecycle representation that jointly fixes the comparison domain and distinguishes changes in operational states, population support, and exposure. In summary, regulation and engineering practice already address AI system multiplicity for several purposes, but not through the common lifecycle representation required by the AI population governance problem defined here. We address it by developing our approach as follows.

%%%%%%%%%%%%%%%%%%%%%%%%%%%%%%%%%%%%%%%%%%%%%%%%%%%%%
%%%%%%%%%%%%%%%%%%%%%%%%%%%%%%%%%%%%%%%%%%%%%%%%%%%%%
\section{\StepI: An Account of AI Identity Across Space and Time}
\label{section:category_account}
We start with the trustworthiness-based and categorical accounts of AI identity introduced in \citep{Ferrario2025TrustworthinessMetaphysics,Ferrario2026CategoryIdentity} to address the \emph{meaningful} comparability and identity of AI systems. Here, we do not reproduce the full formalization and its categorical extension. Instead, we recall its primitives and a few constructs required for the sequel, and the weak and strong identity relations defined on them. The primitive components of the formalization are summarized in Table~\ref{table:formal_primitives} in the Appendix.

%%%%%%%%%%%%%%%%%%%%%%%%%%%%%%%%%%%%%%%%%%%%%%%%%%%%%%%%%%%
\vspace{1em}
\noindent\textbf{AI System Types.}
The starting point of the AI identity formalization is a fixed level of description at which system identity is assessed. This level of description is given by the AI system \emph{type}:

\begin{definition}[AI system type \citep{Ferrario2025TrustworthinessMetaphysics,Ferrario2026CategoryIdentity}]
\label{def:AI_type}
An \emph{AI system type datum} is a triple
\(\diam=(F,P,L_P)\), where \(F\) is a techno-function,
\(P=\{p_1,\ldots,p_n\}\) is a trustworthiness profile for systems
realizing \(F\), and \(L_P:I_P\to[K]\) is a
trustworthiness level function. Unless otherwise stated,
\(I_P=[0,1]^n\) and \([K]=\{1,\ldots,K\}\).
\end{definition}

The techno-function \(F\) specifies the technical capability at the chosen level of abstraction. For instance, in the medical-imaging case of Section~\ref{subsection:AI_healthcare_monitoring}, \(F\) may be the capability of detecting and localizing neurological lesions from magnetic resonance images. While the intended clinical use provides the context in which this capability is deployed, \(F\) identifies the technical capability whose correct realization is being assessed. The profile \(P\) specifies the trustworthiness dimensions relevant to that realization together with the measurement, normalization, evidential, and aggregation conventions through which they are evaluated \citep{ferrario2026methodology}---see Appendix~\ref{appendix:metric_affinity} for a concise exposition.\footnote{As different metrics for the same dimensions are not, in general, comparable, their choice is a key governance step. Moreover, the numerical representation fixed by \(P\) determines the geometry on which convex composition and its objects are defined. We comment on these points in Appendix~\ref{appendix:metric_affinity}.}
Here, we use \emph{AI trustworthiness} in a governance-oriented sense: an AI system is trustworthy to the extent that, for a given intended purpose and context of use, it satisfies the socio-technical requirements documented in its design and governance \citep{ferrario2026methodology}. For an overview of the extensive literature on \emph{Trustworthy AI} and trustworthiness dimensions, we refer to recent reviews \citep{mehrotra2026understanding,kemmerzell2025towards,kaur2022trustworthy,li2023trustworthy,diaz2023connecting} showing that, depending on the application, these dimensions may include predictive performance, robustness, safety, fairness, documentation transparency, explainability, cybersecurity, or human oversight. The quantifications of these dimensions according to the requirements specified by \(P\) are captured by \emph{the quantified profile} of an AI system, namely, a point \(q=(q_1,\ldots,q_n)\in I_P\). The function \(L_P\) assigns these quantified profiles to a finite collection of governance-relevant levels \citep{ferrario2026methodology}. It may, for instance, be a stepwise rule specified by inequalities on selected coordinates based on expert judgments or learned from historical data \citep{ferrario2026methodology}. We show three examples of \(L_P\) in Figure~\ref{figure:LP_examples_1d_2d_3d}. \(L_P\) becomes an AI governance instrument for summarizing the level of \emph{correct functioning} of an AI system, given the variability in performance across its different trustworthiness dimensions over time \citep{ferrario2026methodology} and a key concept for introducing AI identity \citep{Ferrario2025TrustworthinessMetaphysics,Ferrario2026CategoryIdentity}. As an AI audit object, it is validated either as an expert-based rule or empirically \citep{ferrario2026methodology}. Finally,  unlike \(F\), \(P\), and \(L_P\), the vectors \(q\) are \emph{not} part of the fixed datum \(\diam\): they may change across copies, deployments, variants, and times. This distinction will be important below: \(\diam\) fixes the AI identity framework, while \(q\) records where an AI system lies within that framework.

%%%%%%%%%%%%%%%%%%%%%%%%%%%%%%%%%%%%%%%%%%%%%%%%%%%
% ONE-, TWO-, AND THREE-DIMENSIONAL EXAMPLES OF L_P
%%%%%%%%%%%%%%%%%%%%%%%%%%%%%%%%%%%%%%%%%%%%%%%%%%%
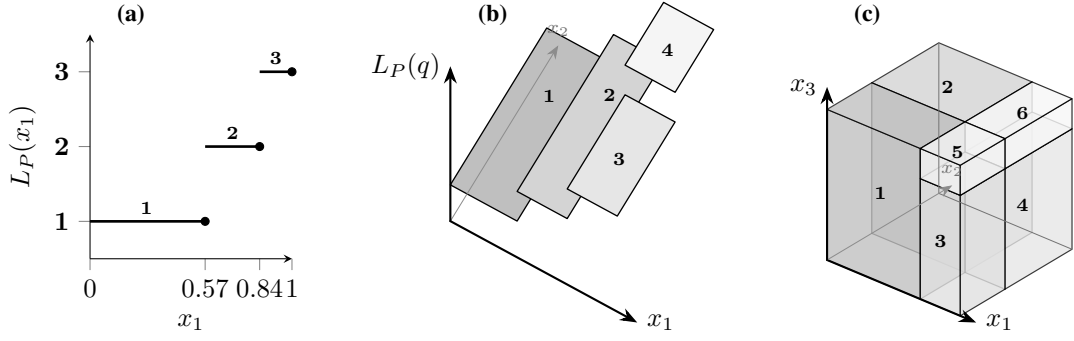
\begin{figure}[t]
\centering

\tdplotsetmaincoords{20}{30}

\tikzset{
  paneltitle/.style={font=\small\bfseries},
  lab/.style={font=\scriptsize},
  axis3d/.style={-{Stealth[length=2.3mm]}, line width=0.8pt},
  axis3dsoft/.style={-{Stealth[length=1.8mm]}, line width=0.35pt, draw=black!45},
  x2lab/.style={font=\scriptsize, text=black!55},
  plateau/.style={draw=black, line width=0.5pt},
  cubeface/.style={draw=black, line width=0.45pt},
  levellab/.style={font=\scriptsize\bfseries},
  hiddenedge/.style={
    draw=black!55,
    line width=0.30pt,
    line cap=round
  }
}

%%%%%%%%%%%%%%%%%%%%%%%%%%%%%%%%%%%%%%%%%%%%%%%%
% GLOBAL PANEL TITLES
%%%%%%%%%%%%%%%%%%%%%%%%%%%%%%%%%%%%%%%%%%%%%%%%
\makebox[\textwidth][c]{%
\begin{minipage}{0.28\textwidth}
\centering
\small\bfseries (a)
\end{minipage}%
% \hspace{0.01\textwidth}%
\begin{minipage}{0.30\textwidth}
\centering
\small\bfseries (b)
\end{minipage}%
% \hspace{0.018\textwidth}%
\begin{minipage}{0.30\textwidth}
\centering
\small\bfseries (c)
\end{minipage}%
}

\vspace{-1em}

\makebox[\textwidth][c]{%
%%%%%%%%%%%%%%%%%%%%%%%%%%%%%%%%%%%%%%%%%%%%%%%%
% PANEL (a): 1D EXAMPLE
%%%%%%%%%%%%%%%%%%%%%%%%%%%%%%%%%%%%%%%%%%%%%%%%
\begin{minipage}[t]{0.28\textwidth}
\centering
\begin{tikzpicture}
\begin{axis}[
  width=0.92\linewidth,
  height=4.55cm,
  xmin=0, xmax=1,
  ymin=0.5, ymax=3.5,
  xlabel={$x_1$},
  ylabel={$L_P(x_1)$},
  xtick={0,0.57,0.84,1},
  ytick={1,2,3},
  yticklabels={$\mathbf{1}$,$\mathbf{2}$,$\mathbf{3}$},
  axis lines=left,
  tick align=outside,
  enlargelimits=false
]

% Step function
\addplot[line width=1pt] coordinates {(0.00,1) (0.57,1)};
\addplot[line width=1pt] coordinates {(0.57,2) (0.84,2)};
\addplot[line width=1pt] coordinates {(0.84,3) (1.00,3)};

% Mark transition points
\addplot[
  only marks,
  mark=*,
  mark size=1.5pt
] coordinates {
  (0.57,1)
  (0.84,2)
  (1.00,3)
};

% Region labels
\node[font=\scriptsize\bfseries]
  at (axis cs:0.28,1.18) {$\mathbf{1}$};

\node[font=\scriptsize\bfseries]
  at (axis cs:0.705,2.18) {$\mathbf{2}$};

\node[font=\scriptsize\bfseries]
  at (axis cs:0.92,3.18) {$\mathbf{3}$};

\end{axis}
\end{tikzpicture}
\end{minipage}%
\hspace{0.018\textwidth}%

%%%%%%%%%%%%%%%%%%%%%%%%%%%%%%%%%%%%%%%%%%%%%%%%
% PANEL (b): 2D EXAMPLE
%%%%%%%%%%%%%%%%%%%%%%%%%%%%%%%%%%%%%%%%%%%%%%%%
\begin{minipage}[t]{0.30\textwidth}
\centering

\begin{tikzpicture}[
  tdplot_main_coords,
  scale=2.55,
  line join=round,
  line cap=round
]

% Schematic heights corresponding to levels 1,2,3,4
\def\zOne{0.55}
\def\zTwo{1.00}
\def\zThr{1.45}
\def\zFou{1.90}
\def\zAxisTop{2.35}

% C1: x1 < 0.4
\filldraw[plateau,fill=black!25]
  (0,0,\zOne) --
  (0.4,0,\zOne) --
  (0.4,1,\zOne) --
  (0,1,\zOne) -- cycle;

% C2: 0.4 <= x1 < 0.7
\filldraw[plateau,fill=black!17]
  (0.4,0,\zTwo) --
  (0.7,0,\zTwo) --
  (0.7,1,\zTwo) --
  (0.4,1,\zTwo) -- cycle;

% C3: x1 >= 0.7 and x2 < 0.6
\filldraw[plateau,fill=black!10]
  (0.7,0,\zThr) --
  (1,0,\zThr) --
  (1,0.6,\zThr) --
  (0.7,0.6,\zThr) -- cycle;

% C4: x1 >= 0.7 and x2 >= 0.6
\filldraw[plateau,fill=black!4]
  (0.7,0.6,\zFou) --
  (1,0.6,\zFou) --
  (1,1,\zFou) --
  (0.7,1,\zFou) -- cycle;

% Region labels
\node[levellab] at (0.20,0.68,\zOne) {$\mathbf{1}$};
\node[levellab] at (0.55,0.70,\zTwo) {$\mathbf{2}$};
\node[levellab] at (0.85,0.28,\zThr) {$\mathbf{3}$};
\node[levellab] at (0.85,0.78,\zFou) {$\mathbf{4}$};

% Axes
\draw[axis3d]
  (0,0,0) -- (1.12,0,0)
  node[anchor=west] {$x_1$};

\draw[axis3dsoft]
  (0,0,0) -- (0,1.12,0)
  node[x2lab,anchor=south] {$x_2$};

\draw[axis3d]
  (0,0,0) -- (0,0,\zAxisTop)
  node[anchor=east] {$L_P(q)$};

\end{tikzpicture}
\end{minipage}%
\hspace{0.018\textwidth}%

%%%%%%%%%%%%%%%%%%%%%%%%%%%%%%%%%%%%%%%%%%%%%%%%
% PANEL (c): 3D EXAMPLE
%%%%%%%%%%%%%%%%%%%%%%%%%%%%%%%%%%%%%%%%%%%%%%%%

% Viewing angles for panel (c) only
\newcommand{\CubeTheta}{60}
\newcommand{\CubePhi}{40}
\tdplotsetmaincoords{\CubeTheta}{\CubePhi}

\begin{minipage}[t]{0.20\textwidth}
\centering

\begin{tikzpicture}[
  tdplot_main_coords,
  scale=2.3,
  line join=round,
  line cap=round
]

%%%%%%%%%%%%%%%%%%%%%%%%%%%%%%%%%%%%%%%%%%%%%%%%%%%%%
% Draw cuboids: back ones first, then front ones
% Visible faces used: y = ymin, x = xmax, z = zmax
%%%%%%%%%%%%%%%%%%%%%%%%%%%%%%%%%%%%%%%%%%%%%%%%%%%%%

%%%%%%%%%%%%%%%%%%%%%%%%%%%%%%%%%%%%%%%%%%%%%%%%%%%%%
% C2
% x1 in [0,0.7], x2 in [0.4,1], x3 free
%%%%%%%%%%%%%%%%%%%%%%%%%%%%%%%%%%%%%%%%%%%%%%%%%%%%%
\filldraw[cubeface,fill=black!18,opacity=0.75]
  (0,0.4,0) --
  (0.7,0.4,0) --
  (0.7,0.4,1) --
  (0,0.4,1) -- cycle;

\filldraw[cubeface,fill=black!18,opacity=0.75]
  (0.7,0.4,0) --
  (0.7,1,0) --
  (0.7,1,1) --
  (0.7,0.4,1) -- cycle;

\filldraw[cubeface,fill=black!18,opacity=0.75]
  (0,0.4,1) --
  (0.7,0.4,1) --
  (0.7,1,1) --
  (0,1,1) -- cycle;

\node[levellab]
  at (0.30,0.72,0.94)
  {$\mathbf{2}$};

%%%%%%%%%%%%%%%%%%%%%%%%%%%%%%%%%%%%%%%%%%%%%%%%%%%%%
% C4
% x1 in [0.7,1], x2 in [0.4,1], x3 in [0,0.8]
%%%%%%%%%%%%%%%%%%%%%%%%%%%%%%%%%%%%%%%%%%%%%%%%%%%%%
\filldraw[cubeface,fill=black!10,opacity=0.78]
  (0.7,0.4,0) --
  (1,0.4,0) --
  (1,0.4,0.8) --
  (0.7,0.4,0.8) -- cycle;

\filldraw[cubeface,fill=black!10,opacity=0.78]
  (1,0.4,0) --
  (1,1,0) --
  (1,1,0.8) --
  (1,0.4,0.8) -- cycle;

\filldraw[cubeface,fill=black!10,opacity=0.78]
  (0.7,0.4,0.8) --
  (1,0.4,0.8) --
  (1,1,0.8) --
  (0.7,1,0.8) -- cycle;

\node[levellab]
  at (0.88,0.70,0.38)
  {$\mathbf{4}$};

%%%%%%%%%%%%%%%%%%%%%%%%%%%%%%%%%%%%%%%%%%%%%%%%%%%%%
% C6
% x1 in [0.7,1], x2 in [0.4,1], x3 in [0.8,1]
%%%%%%%%%%%%%%%%%%%%%%%%%%%%%%%%%%%%%%%%%%%%%%%%%%%%%
\filldraw[cubeface,fill=black!3,opacity=0.82]
  (0.7,0.4,0.8) --
  (1,0.4,0.8) --
  (1,0.4,1) --
  (0.7,0.4,1) -- cycle;

\filldraw[cubeface,fill=black!3,opacity=0.82]
  (1,0.4,0.8) --
  (1,1,0.8) --
  (1,1,1) --
  (1,0.4,1) -- cycle;

\filldraw[cubeface,fill=black!3,opacity=0.82]
  (0.7,0.4,1) --
  (1,0.4,1) --
  (1,1,1) --
  (0.7,1,1) -- cycle;

\node[levellab]
  at (0.88,0.70,0.99)
  {$\mathbf{6}$};

%%%%%%%%%%%%%%%%%%%%%%%%%%%%%%%%%%%%%%%%%%%%%%%%%%%%%
% C1
% x1 in [0,0.7], x2 in [0,0.4], x3 free
%%%%%%%%%%%%%%%%%%%%%%%%%%%%%%%%%%%%%%%%%%%%%%%%%%%%%
\filldraw[cubeface,fill=black!26,opacity=0.75]
  (0,0,0) --
  (0.7,0,0) --
  (0.7,0,1) --
  (0,0,1) -- cycle;

\filldraw[cubeface,fill=black!26,opacity=0.75]
  (0.7,0,0) --
  (0.7,0.4,0) --
  (0.7,0.4,1) --
  (0.7,0,1) -- cycle;

\filldraw[cubeface,fill=black!26,opacity=0.75]
  (0,0,1) --
  (0.7,0,1) --
  (0.7,0.4,1) --
  (0,0.4,1) -- cycle;

\node[levellab]
  at (0.30,0.12,0.50)
  {$\mathbf{1}$};

%%%%%%%%%%%%%%%%%%%%%%%%%%%%%%%%%%%%%%%%%%%%%%%%%%%%%
% C3
% x1 in [0.7,1], x2 in [0,0.4], x3 in [0,0.8]
%%%%%%%%%%%%%%%%%%%%%%%%%%%%%%%%%%%%%%%%%%%%%%%%%%%%%
\filldraw[cubeface,fill=black!12,opacity=0.80]
  (0.7,0,0) --
  (1,0,0) --
  (1,0,0.8) --
  (0.7,0,0.8) -- cycle;

\filldraw[cubeface,fill=black!12,opacity=0.80]
  (1,0,0) --
  (1,0.4,0) --
  (1,0.4,0.8) --
  (1,0,0.8) -- cycle;

\filldraw[cubeface,fill=black!12,opacity=0.80]
  (0.7,0,0.8) --
  (1,0,0.8) --
  (1,0.4,0.8) --
  (0.7,0.4,0.8) -- cycle;

\node[levellab]
  at (0.75,0.12,0.35)
  {$\mathbf{3}$};

%%%%%%%%%%%%%%%%%%%%%%%%%%%%%%%%%%%%%%%%%%%%%%%%%%%%%
% C5
% x1 in [0.7,1], x2 in [0,0.4], x3 in [0.8,1]
%%%%%%%%%%%%%%%%%%%%%%%%%%%%%%%%%%%%%%%%%%%%%%%%%%%%%
\filldraw[cubeface,fill=black!5,opacity=0.83]
  (0.7,0,0.8) --
  (1,0,0.8) --
  (1,0,1) --
  (0.7,0,1) -- cycle;

\filldraw[cubeface,fill=black!5,opacity=0.83]
  (1,0,0.8) --
  (1,0.4,0.8) --
  (1,0.4,1) --
  (1,0,1) -- cycle;

\filldraw[cubeface,fill=black!5,opacity=0.83]
  (0.7,0,1) --
  (1,0,1) --
  (1,0.4,1) --
  (0.7,0.4,1) -- cycle;

\node[levellab]
  at (0.88,0.12,0.99)
  {$\mathbf{5}$};

%%%%%%%%%%%%%%%%%%%%%%%%%%%%%%%%%%%%%%%%%%%%%%%%%%%%%
% MISSING / HIDDEN REAR EDGES
%%%%%%%%%%%%%%%%%%%%%%%%%%%%%%%%%%%%%%%%%%%%%%%%%%%%%

% Rear-bottom edge behind C2
\draw[hiddenedge]
  (0,1,0) -- (0.7,1,0);

% Rear-bottom edge behind C4
\draw[hiddenedge]
  (0.7,1,0) -- (1,1,0);

% Far rear-left vertical outer edge
\draw[hiddenedge]
  (0,1,0) -- (0,1,1);

%%%%%%%%%%%%%%%%%%%%%%%%%%%%%%%%%%%%%%%%%%%%%%%%%%%%%
% AXES
%%%%%%%%%%%%%%%%%%%%%%%%%%%%%%%%%%%%%%%%%%%%%%%%%%%%%
\draw[axis3d]
  (0,0,0) -- (1.12,0,0)
  node[anchor=west] {$x_1$};

\draw[axis3dsoft]
  (0,0,0) -- (0,1.12,0)
  node[x2lab,anchor=south] {$x_2$};

\draw[axis3d]
  (0,0,0) -- (0,0,1.15)
  node[anchor=east] {$x_3$};

\end{tikzpicture}
\end{minipage}%
}

\caption{
Examples of functions \(L_P\).
(a) A one-dimensional stepwise example with profile coordinate \(x_1\) and thresholds at \(0.57\) and \(0.84\).
(b) A two-dimensional example represented as a stepwise-constant graph over \([0,1]^2\), with four rectangular level regions.
(c) A three-dimensional example represented by the partition of the cube \([0,1]^3\) into six rectangular level regions \(C_1,\dots,C_6\).
}
\label{figure:LP_examples_1d_2d_3d}
\end{figure}
%%%%%%%%%%%%%%%%%%%%%%%%%%%%%%%%%%%%%%%%%%%%%%%%%%%

%%%%%%%%%%%%%%%%%%%%%%%%%%%%%%%%%%%%%%
\vspace{1em}
\noindent\textbf{Abstract AI System States}
% \label{subsection:Abstract_AI_system_states}
The quantified profile \(q\) becomes a key concept of the categorical extension of the AI identity formalization by pairing it with its image under \(L_P\). The results of this pairing are the AI system states:

\begin{definition}[\(\diam\)-relative AI system states, \citep{Ferrario2026CategoryIdentity}]
\label{def:profile_state}
For a fixed type datum \(\diam=(F,P,L_P)\), the \(\diam\)-relative AI state space is \(S_\diam:=\{(q,L_P(q))\mid q\in I_P\}\).
For a state \(x=(q_x,L_P(q_x))\in S_\diam\), define its level by
\(\tau_\diam(x):=L_P(q_x)\). For \(k\in[K]\), let 
\[
C_k:=L_P^{-1}(k)
=
\{q\in I_P:L_P(q)=k\},
\quad
S_{\diam,k}:=\{(q,k):q\in C_k\}.
\] 
\end{definition}
Then, \emph{AI states are abstract representations of the possible quantifications of trustworthiness profile \(P\) and the corresponding trustworthiness levels that AI systems of a given type \(\diam\) may occupy across deployments and over time}. In general, these states could be occupied by different AI system instantiations---or by none---such as in the case of AI systems deployed in different hospitals, using different software installations, or models, whenever despite these differences they receive the same quantified assessment under \(P\) at the selected level of precision. Conversely, one AI system instantiation may occupy different states as its operating conditions evolve over time. Each fiber \(S_{\diam,k}\) collects states \(x\in S_\diam\) that agree at the trustworthiness level while differing in their quantified profiles \(q_x\).

%%%%%%%%%%%%%%%%%%%%%%%%%%%%%%%%%%%%%%%%%%%%%%%%%%%
\vspace{1em}
\noindent\textbf{AI system transformations and trustworthiness-preserving reachability.}
% \label{subsection:Transformations_AI_states}
To represent change of states and systems, fix a vocabulary \(\mathcal U_\diam\) of admissible primitive lifecycle transformations between AI states in \(S_\diam\) \citep{Ferrario2026CategoryIdentity}. These may be
\emph{intervention-driven}, such as retraining, recalibration, data refresh, rollback, model replacement, or changes to human oversight, or \emph{environmentally induced}, such as distribution shift, changes in populations, sensors, data pipelines, workflows, infrastructure, or dependencies. The distinction concerns the provenance of a transformation and does not determine its direction or effect on the
trustworthiness profile. Each \(r\in\mathcal U_\diam\) is represented by a relation
\(R_r\subseteq S_\diam\times S_\diam,\) where \(xR_ry\) means that \(y\) is an admissible state outcome of a primitive transformation of type \(r\) starting from \(x\) \citep{Ferrario2026CategoryIdentity}. The term
\emph{primitive} is relative to the selected level of description: for instance, retraining may constitute one transformation in a coarse model and several transformations in a finer one. \emph{Admissibility} is likewise context dependent and may reflect technical feasibility, available data and resources or governance requirements \citep{Ferrario2026CategoryIdentity}. Panel~(b) of Figure~\ref{figure:datum_state_sys} illustrates this relational structure. Primitive transformations generate finite \emph{lifecycle paths}
\(
x=x_0\xrightarrow{r_1}x_1
\xrightarrow{r_2}\cdots
\xrightarrow{r_m}x_m=y,\)
\((x_{i-1},x_i)\in R_{r_i}.\)
If there exists at least one path from \(x\) to \(y\), then we say that \(y\) is \emph{reachable} from \(x\). A path is \emph{trustworthiness preserving} when all its states belong to one common trustworthiness level fiber. This is an important definition:

\begin{definition}[Trustworthiness-preserving reachability, \citep{Ferrario2026CategoryIdentity}]
\label{def:lifecycle_reachability}
For \(x,y\in S_\diam\), write
\(x\preceq_\tau y\)
when there exists a trustworthiness-preserving lifecycle path from \(x\)
to \(y\). The relation \(\preceq_\tau\) is a preorder: reflexivity is
given by the empty path and transitivity by path concatenation.
\end{definition}

Thus, \(\preceq_\tau\) records that at least one admissible lifecycle route connects
the two states while preserving their trustworthiness level throughout. Importantly, \emph{reachability is non-trivial}: two states in the same trustworthiness level fiber need not be connected by any admissible lifecycle path. Whether \(x\preceq_\tau y\) holds depends jointly on the primitive transformation vocabulary \(\mathcal U_\diam\), the constraints encoded by its relations \(R_r\), and the existence of a finite path that remains within one trustworthiness level.\footnote{Institutional constraints may also delimit feasible operating regions within a single fiber \(C_k\), reflecting, for instance, provider-defined and validated deployment regimes. Even states sharing the same trustworthiness level need not be connected by an admissible lifecycle path---see  Appendix~\ref{appendix:reachability_example}.} We give an explicit example in Appendix~\ref{appendix:reachability_example}. The categorification of this idea of reachability is: 
\begin{definition}[AI system state category, \citep{Ferrario2026CategoryIdentity}]
\label{def:state_category}
The category \(\Sys_\diam\) is the thin category associated with  \((S_\diam,\preceq_\tau)\):
\[
\Hom_{\Sys_\diam}(x,y)
=
\begin{cases}
\{\ast_{x,y}\}, & x\preceq_\tau y,\\
\varnothing, & \text{otherwise}.
\end{cases}
\]
\end{definition}

The unique morphism \(\ast_{x,y}\) retains only the existence and direction of trustworthiness-preserving reachability. It forgets the number and labels of transformations, intermediate states, path length, and lifecycle provenance. \(\Sys_\diam\) simply records whether one state can reach another while preserving trustworthiness levels; it does not record the particular lifecycle path taken or whether any concrete AI system actually followed that path.  Figure~\ref{figure:datum_state_sys} summarizes graphically the construction of \(\Sys_\diam\) from \((S_\diam,\preceq_\tau)\) and \(\diam\). We close
\StepI~ by discussing AI identity \citep{Ferrario2026CategoryIdentity}.

\vspace{1em}
\noindent\textbf{Weak and strong AI state identity.}
State reachability encoded within the category \(\Sys_\diam\) allows us to introduce two different levels of AI identity \citep{Ferrario2026CategoryIdentity} that replicate and generalize the identity criteria expressed in propositional form in \citep{Ferrario2025TrustworthinessMetaphysics}.

\begin{definition}[Weak and strong AI identity, \citep{Ferrario2026CategoryIdentity}]
\label{def:weak_strong_identity}
For \(x,y\in S_\diam\), define
\begin{align*}
&x\simeq^{\mathrm{wk}}_\diam y
\Longleftrightarrow
\tau_\diam(x)=\tau_\diam(y), \quad\text{(weak AI identity)}\\
&x\simeq^{\mathrm{str}}_\diam y\Longleftrightarrow x\preceq_\tau y\ \text{and}\
y\preceq_\tau x.\quad\text{(strong AI identity)}
\end{align*}
Equivalently,
\(x\simeq^{\mathrm{str}}_\diam y
\Longleftrightarrow x\cong_{\Sys_\diam}y,\) that is, strong AI identity is isomorphism in \(\Sys_\diam\).
\end{definition}

Weak identity of AI states in \(\Sys_\diam\) requires only membership in the same trustworthiness level fiber: all states in any \(S_{\diam,k},\) \(k\in[K],\) are weakly identical. Strong identity of AI states additionally requires mutual trustworthiness-preserving reachability. This distinction will matter for AI composability, as we will see in Section~\ref{section:convex_composition}.
% \footnote{The full categorical account in \citep{Ferrario2026CategoryIdentity} also includes history functors; here we use only \(S_\diam\), \(\preceq_\tau\), \(\Sys_\diam\), and the resulting identity relations.} 
In summary, \StepI~ fixes the comparison domain within which an AI system population can be defined. AI instantiations assessed under a common datum \(\diam\) can be represented as states in \(S_\diam\), while weak and strong identity specify different forms of identity preservation across those states and their lifecycle changes. 

%%%%%%%%%%%%%%%%%%%%%%%%%%%%%%%%%%%%%%%%%%%%%%%%
% TYPE DATUM, STATE SPACE, AND THIN CATEGORY
%%%%%%%%%%%%%%%%%%%%%%%%%%%%%%%%%%%%%%%%%%%%%%%%
% Requires in the preamble:
% \usetikzlibrary{decorations.markings}

\begin{figure}[h!]
\centering

%%%%%%%%%%%%%%%%%%%%%%%%%%%%%%%%%%%%%%%%%%%%%%%%%%%%%
% USER CONTROL: vertical compression only
%%%%%%%%%%%%%%%%%%%%%%%%%%%%%%%%%%%%%%%%%%%%%%%%%%%%%
\def\StepIYScale{1.00}   % 1.00 = original height; 0.78 = ~22% shorter

\makebox[\textwidth][c]{%
\begin{tikzpicture}[
    >=Latex,
    font=\small,
    line cap=round,
    line join=round,
    paneltitle/.style={font=\small\bfseries},
    box/.style={draw, rounded corners=2pt, line width=0.6pt, inner sep=3pt},
    arr/.style={-{Latex[length=2.0mm,width=1.0mm]}, line width=0.10pt},
    rel/.style={-{Latex[length=2mm,width=1.0mm]}, line width=0.3pt},
    sysrel/.style={
        draw=black,
        line width=1.6pt,
        postaction={decorate},
        decoration={
            markings,
            mark=at position 0.55 with
            {\arrow{Latex[length=2.3mm,width=1.6mm]}}
        }
    },
    pt/.style={circle, fill=black, inner sep=1.2pt},
    lab/.style={font=\small},
    fiblab/.style={font=\small},
    plane/.style={draw=black, line width=0.2pt, dotted},
    sep/.style={draw=black!55, line width=0.35pt}
]

%%%%%%%%%%%%%%%%%%%%%%%%%%%%%%%%%%%%%%%%%%%%%%%%
% PANEL TITLES
%%%%%%%%%%%%%%%%%%%%%%%%%%%%%%%%%%%%%%%%%%%%%%%%
\node[paneltitle] at (1.35,{\StepIYScale*5.20}) {(a) Type datum};
\node[paneltitle] at (6.45,{\StepIYScale*5.20}) {(b) State space $S_\diam$};
\node[paneltitle] at (13.95,{\StepIYScale*5.20}) {(c) Thin category $\Sys_\diam$};

%%%%%%%%%%%%%%%%%%%%%%%%%%%%%%%%%%%%%%%%%%%%%%%%
% VERTICAL SEPARATORS
%%%%%%%%%%%%%%%%%%%%%%%%%%%%%%%%%%%%%%%%%%%%%%%%
\draw[sep]
  (2.85,{\StepIYScale*0.45})
  --
  (2.85,{\StepIYScale*4.10});

\draw[sep]
  (10.70,{\StepIYScale*0.45})
  --
  (10.70,{\StepIYScale*4.10});

%%%%%%%%%%%%%%%%%%%%%%%%%%%%%%%%%%%%%%%%%%%%%%%%
% PANEL 1: TYPE DATUM
%%%%%%%%%%%%%%%%%%%%%%%%%%%%%%%%%%%%%%%%%%%%%%%%
\begin{scope}[shift={(0,0)},scale=0.86,transform shape]

    \node[font=\large] (diamond)
        at (1.8,{\StepIYScale*2.85})
        {$\diam=(F,P,L_P)$};

\end{scope}

%%%%%%%%%%%%%%%%%%%%%%%%%%%%%%%%%%%%%%%%%%%%%%%%
% PANEL 2: S_diam
%%%%%%%%%%%%%%%%%%%%%%%%%%%%%%%%%%%%%%%%%%%%%%%%
\begin{scope}[shift={(2.9,0)},scale=1.10,transform shape]

    % planes
    \coordinate (A1) at (0.4,{\StepIYScale*0.85});
    \coordinate (B1) at (4.2,{\StepIYScale*0.85});
    \coordinate (C1) at (5.0,{\StepIYScale*1.35});
    \coordinate (D1) at (1.2,{\StepIYScale*1.35});

    \coordinate (A2) at (1.2,{\StepIYScale*2.05});
    \coordinate (B2) at (5.0,{\StepIYScale*2.05});
    \coordinate (C2) at (5.8,{\StepIYScale*2.55});
    \coordinate (D2) at (2.0,{\StepIYScale*2.55});

    \coordinate (A3) at (2.0,{\StepIYScale*3.25});
    \coordinate (B3) at (5.8,{\StepIYScale*3.25});
    \coordinate (C3) at (6.6,{\StepIYScale*3.75});
    \coordinate (D3) at (2.8,{\StepIYScale*3.75});

    \draw[plane] (A1)--(B1)--(C1)--(D1)--cycle;
    \draw[plane] (A2)--(B2)--(C2)--(D2)--cycle;
    \draw[plane] (A3)--(B3)--(C3)--(D3)--cycle;

    % fiber labels
    \node[fiblab] at (4.9,{\StepIYScale*0.93})
        {$S_{\diam,1}$};

    \node[fiblab] at (6.00,{\StepIYScale*2.22})
        {$S_{\diam,K-1}$};

    \node[fiblab] at (6.5,{\StepIYScale*3.33})
        {$S_{\diam,K}$};

    \node[lab] at (4.80,{\StepIYScale*1.80})
        {$\vdots$};

    % topic label
    \node[lab, anchor=east]
        at (0.70,{\StepIYScale*2.22})
        {$S_\diam$};

    % states
    \coordinate (x1) at (1.55,{\StepIYScale*1.02});
    \coordinate (x2) at (3.10,{\StepIYScale*1.15});
    \coordinate (x3) at (3.55,{\StepIYScale*2.28});
    \coordinate (x4) at (5.55,{\StepIYScale*3.50});
    \coordinate (x5) at (3.50,{\StepIYScale*3.67});

    \node[pt,label={[lab,left=-1pt]$x_1$}] at (x1) {};
    \node[pt,label={[lab,below right=-1pt]$x_2$}] at (x2) {};
    \node[pt,label={[lab,left=1pt]$x_3$}] at (x3) {};
    \node[pt,label={[lab,right]$x_4$}] at (x4) {};
    \node[pt,label={[lab,below left]$x_5$}] at (x5) {};

    % primitive relations
    \draw[rel] (x1) --
        node[lab,above left=1.5pt] {$r_{12}$} (x2);

    \draw[rel] (x1) to[bend left=18]
        node[lab,above right=-1pt] {$r_{12}'$} (x2);

    \draw[rel] (x1) to[bend right=18]
        node[lab,below] {$r_{12}''$} (x2);

    \draw[rel] (x2) --
        node[lab,right] {$r_{23}$} (x3);

    \draw[rel] (x3) --
        node[lab,left=2pt] {$r_{34}$} (x4);

    \draw[rel] (x4) to[bend left=18]
        node[lab,right=4pt] {$r_{43}$} (x3);

    \draw[rel] (x5) --
        node[lab,above=-2pt] {$r_{54}$} (x4);

    \draw[rel] (x4) to[bend left=22]
        node[lab,above=-2.5pt] {$r_{45}$} (x5);

\end{scope}

%%%%%%%%%%%%%%%%%%%%%%%%%%%%%%%%%%%%%%%%%%%%%%%%
% PANEL 3: Sys_diam
%%%%%%%%%%%%%%%%%%%%%%%%%%%%%%%%%%%%%%%%%%%%%%%%
\begin{scope}[shift={(11.2,0)},scale=1.10,transform shape]

    % topic label
    \node[lab, anchor=east]
        at (0.50,{\StepIYScale*2.22})
        {$\Sys_\diam$};

    % states
    \coordinate (y1) at (0.10,{\StepIYScale*1.08});
    \coordinate (y2) at (1.40,{\StepIYScale*1.18});
    \coordinate (y3) at (1.90,{\StepIYScale*2.40});
    \coordinate (y5) at (1.80,{\StepIYScale*3.67});
    \coordinate (y4) at (4.05,{\StepIYScale*3.58});

    \node[pt,label={[lab,below left=-1pt]$x_1$}] at (y1) {};
    \node[pt,label={[lab,below right=-1pt]$x_2$}] at (y2) {};
    \node[pt,label={[lab,left=-1pt]$x_3$}] at (y3) {};
    \node[pt,label={[lab,left=-1pt]$x_5$}] at (y5) {};
    \node[pt,label={[lab,right=-1pt]$x_4$}] at (y4) {};

    \node[lab, anchor=east]
        at (1.30,{\StepIYScale*1.50})
        {\(x_1\preceq_\tau x_2\)};

    \node[lab, anchor=east]
        at (3.6,{\StepIYScale*4.20})
        {\(x_5\simeq^{\mathrm{str}}_\diam x_4\)};

    % thin-category arrows
    \draw[sysrel] (y1) -- (y2);
    \draw[sysrel] (y5) to[bend left=16] (y4);
    \draw[sysrel] (y4) to[bend left=16] (y5);

\end{scope}

\end{tikzpicture}%
}

\caption{
Diagrammatic summary of the categorical construction in
\citep{Ferrario2026CategoryIdentity}.
(a) The type datum \(\diam=(F,P,L_P)\) specifies the techno-function,
trustworthiness profile, and trustworthiness level function.
(b) The state space \(S_\diam\) is partitioned into level fibers
\(S_{\diam,k}\), with primitive lifecycle transformations represented by
relations between states, possibly across fibers.
(c) The thin category \(\Sys_\diam\) retains only
trustworthiness-preserving reachability; the thicker arrows indicate the
structure underlying weak and strong AI state identity.
}
\label{figure:datum_state_sys}
\end{figure}
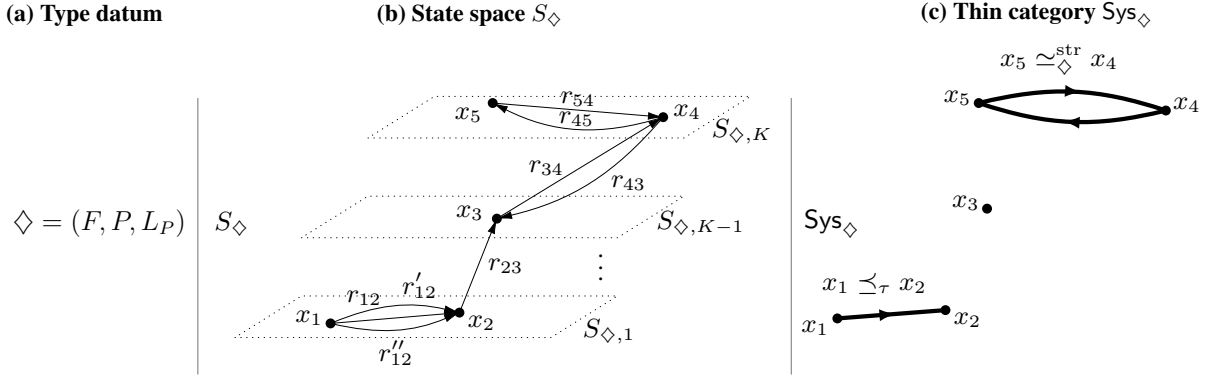

%%%%%%%%%%%%%%%%%%%%%%%%%%%%%%%%%%%%%%%%%%%%%%%%
%%%%%%%%%%%%%%%%%%%%%%%%%%%%%%%%%%%%%%%%%%%%%%%%
\section{\StepII: Convex Composition of AI System States}
\label{section:convex_composition}
We can now address the second and novel step of our approach: how
configurations of AI states can be composed into AI population governance objects. The construction operates on the quantified profiles \(q\in I_P\) supplied by the protocol \(P\) discussed in \StepI. The protocol \(P\) needs to provide numerical coordinates on which such composition is defined; therefore, discrete or ordinal measures require an appropriate numerical representation in \(I_P\). For the interested reader, Appendix~\ref{appendix:metric_affinity} summarizes representative
trustworthiness dimensions and metrics from major surveys and reviews
\citep{kemmerzell2025towards,kaur2022trustworthy,li2023trustworthy,diaz2023connecting}, together with recent work on AI-agent reliability
\citep{rabanser2026towards}.

%%%%%%%%%%%%%%%%%%%%%%%%%%%%%%%%%%%%%%%%%%%%%%%%%%%%%
\subsection{Convex Spaces,  Reachability, and AI Identity Congruence}
\label{subsection:convex_spaces}
\begin{definition}[Convex space]
A \emph{convex space}, equivalently a \emph{barycentric algebra}
\citep{fritz2009convex,zamojska2024barycentric,Stone1949Barycentric}, is a set \(A\)
equipped with binary operations
\(\omega_\alpha:A\times A\to A,\
\alpha\in[0,1],\)
satisfying \(\omega_1(x,y)=x\)\footnote{By skew-commutativity, \(\omega_1(x,y)=x\) implies \(\omega_0(x,y)=\omega_1(y,x)=y\).} and
\begin{align}
&\omega_\alpha(x,x)=x,\quad\text{(idempotency)}\label{eq:bar_idem}\\
&\omega_\alpha(x,y)=\omega_{1-\alpha}(y,x),\quad\text{(skew-commutativity)}\\
&\omega_\beta(\omega_\alpha(x,y),z)
=\omega_{\alpha\beta}
\left(x,\omega_{\frac{\beta(1-\alpha)}
{1-\alpha\beta}}(y,z)
\right),
\quad \alpha\beta\neq1.\quad\text{(skew-associativity)}\label{eq:bar_skew_ass}
\end{align}
% \footnotetext{If \(\alpha\beta=1\), then \(\alpha=\beta=1,\) \(\alpha,\beta\in[0,1]\). In this case, \(\omega_1(x,\omega_\gamma(y,z))=x\) for any \(\gamma\in[0,1]\).}
\end{definition}

The canonical example of a convex space is a convex subset \(C\subseteq V\) of a real vector space \(V\),
with binary operations \(\omega^{C}_\alpha(x,y):=\alpha x+(1-\alpha)y,\) \(x,y\in C\) and \(\alpha\in[0,1].\) Conditions~\eqref{eq:bar_idem}--\eqref{eq:bar_skew_ass} easily follow using the definition of \(\omega^{C}_\alpha\). While Appendix~\ref{appendix:barycentric_basics} introduces other examples of convex spaces, here we apply convexity to AI states straightforwardly:

\begin{proposition}[Convex composition of AI states]
\label{prop:convex_comp_states}
Let \(x,y\in S_{\diam,k}\). If \(C_k\) is convex, define
\begin{align}
\omega_{\alpha,k}(x,y)
:=
\bigl(\alpha q_x+(1-\alpha)q_y,k\bigr),
\qquad \alpha\in[0,1].\label{eq:convex_comp_states}
\end{align}
Then \(\omega_{\alpha,k}(x,y)\in S_{\diam,k}\) for every \(\alpha\in[0,1]\), and the family \(\{\omega_{\alpha,k}\}_{\alpha\in[0,1]}\) equips
\(S_{\diam,k}\) with a convex space structure. Consequently,
\(
\omega_{\alpha,k}(x,y)
\simeq^{\mathrm{wk}}_\diam x
\simeq^{\mathrm{wk}}_\diam y.
\)
\end{proposition}

\begin{proof}
Convexity gives
\(\alpha q_x+(1-\alpha)q_y\in C_k\), hence
\(\bigl(\alpha q_x+(1-\alpha)q_y,k\bigr)\in S_{\diam,k}\), as \(C_k:=L^{-1}_P(k)\). The  identities~\eqref{eq:bar_idem}--\eqref{eq:bar_skew_ass} follow from the corresponding identities for convex combinations in \(\mathbb R^n\).
\end{proof}

Convexity of the preimage sets \(C_k\) of \(L_P\) provides closure of state compositions at a fixed trustworthiness level; the maps \(\{\omega_{\alpha,k}\}_\alpha\) grant the \emph{convex compositionality of AI states}. To study the compatibility of convex compositionality of AI states with the AI identity relations of Definition~\ref{def:weak_strong_identity}, we need to introduce the following definition and simple result.

\begin{definition}[Reachability-monotone composition]
\label{def:reach_monotone}
The convex space structure on \(S_{\diam,k}\) is \emph{reachability monotone} when
for every \(\alpha\in[0,1], x_1\preceq_\tau y_1, x_2\preceq_\tau y_2\) implies \(\omega_{\alpha,k}(x_1,x_2)
\preceq_\tau
\omega_{\alpha,k}(y_1,y_2).
\)
\end{definition}

\begin{proposition}
\label{prop:identity_compatible_composition}
Assume that \(C_k\) is convex and the compositions \(\{\omega_{\alpha,k}\}_\alpha\) on \(S_{\diam,k}\) are reachability monotone for all \(\alpha\in[0,1].\) For
\(x_1,y_1,x_2,y_2\in S_{\diam,k}\) if \(x_1\simeq^{\mathrm{str}}_\diam y_1,~x_2\simeq^{\mathrm{str}}_\diam y_2,\) then \(\omega_{\alpha,k}(x_1,x_2)\simeq^{\mathrm{str}}_\diam\omega_{\alpha,k}(y_1,y_2)\)
for every \(\alpha\in[0,1]\). Hence strong identity is a congruence for
the compositions \(\{\omega_{\alpha,k}\}_\alpha\) on \(S_{\diam,k}.\)
\end{proposition}

\begin{proof}
Strong identity gives reachability in both directions for each pair, as in Definition~\ref{def:weak_strong_identity}. Applying reachability monotonicity as in Definition~\ref{def:reach_monotone} in both directions gives mutual reachability of the corresponding composed states.
\end{proof}

Proposition~\ref{prop:identity_compatible_composition} states a compatibility condition between AI identity and AI composability,  connecting the reachability preorder and strong identity relation on states with their convex compositions. In Appendix~\ref{appendix:homogeneous_transformations}, we provide an explicit local sufficient condition under which this compatibility
condition holds. In what follows, we discuss convexity.

%%%%%%%%%%%%%%%%%%%%%%%%%%%%%%%%%
\subsection{Interpretation of Convex AI Compositionality}
\label{subsection:int_convex_AI_comp}
The motivation for the compositions \(\{\omega_{\alpha,k}\}_\alpha\) comes from the AI population configurations introduced in Section~\ref{section:motivation}. Given a governance mode---e.g., the retrospective audit of the dynamics of an AI population, see Section~\ref{section:discussion}---governance actors first select a finite set of observed or candidate \emph{anchor states}. For simplicity, consider initially two anchors \(x,y\in S_{\diam,k}\). For every \(\alpha\in[0,1]\), \(\omega^{I_P}_{\alpha}(q_x,q_y):=\alpha q_x+(1-\alpha)q_y\) defines the quantified state obtained by interpolation of the two anchors. Convexity of \(C_k\) is the closure condition ensuring that this composition generates possible states from the selected anchors while preserving the trustworthiness level, and therefore weak AI identity, as in Proposition~\ref{prop:convex_comp_states}. As \(\alpha\) varies, these compositions trace the segment joining \(q_x\) and \(q_y\) in \(I_P\): \(\alpha\) specifies the relative contribution of the two anchors to the composition. Where appropriate to the governance mode, \(\alpha\) and \(1-\alpha\) may encode the fractions of a selected population quantity, such as AI instantiations, processed cases, or inference traffic, allocated to the two anchor states. \emph{Thus, this segment is the image of all possible unconstrained allocations of a selected population quantity between the two anchor states.} (Note that the AI population itself remains supported on the anchors \emph{only}.) The segment may be used counterfactually in governance to compare alternative allocations or, when governance treats such allocations as permissible, normatively as an admissible composition region; Section~\ref{section:discussion} develops these interpretations in some detail. For more than two anchors, the construction generalizes straightforwardly, as we will see in  Section~\ref{section:population_governance}.
However, this notion of convex AI compositionality must be distinguished from lifecycle reachability. A state on such a segment need not be occupied by any AI instantiation at a given time, need not lie on a lifecycle path from \(x\) to \(y\), and need not represent the result of recomputing the underlying metrics on a pooled dataset or across deployments.

%%%%%%%%%%%%%%%%%%%%%%%%%%%%%%
\subsection{When Is Convexity of \(C_k\) a Reasonable Assumption?}
\label{subsection:convex_C_k}
Convexity of a trustworthiness level fiber \(C_k\) as in Proposition~\ref{prop:convex_comp_states} is an assumption satisfied by a broad and governance-relevant class of level functions \(L_P\). In particular, regions defined by conjunctions of affine lower and upper bounds,
\(C_k=\left\{q\in I_P:a_\ell\leq c_\ell^\top q\leq b_\ell,\ \ell=1,\ldots,r\right\},\)
are intersections of half-spaces and are therefore convex. The axis-aligned case
\(a_j\leq q_j\leq b_j,\) \(j=1,\ldots,n,\) is natural when governance levels are specified through dimension-wise requirements on performance, robustness, safety, human oversight, or other measured properties. Such rules can encode, for instance, minimum acceptable values or admissible operating ranges without allowing poor performance in one constrained dimension to be compensated arbitrarily by high performance in another. As a result, convexity has a direct governance interpretation: it expresses coherence between AI composability and the trustworthiness partition. Note that all examples of \(L_P\) functions in Figure~\ref{figure:LP_examples_1d_2d_3d} have convex \(C_k\).
However, we do not assume that \emph{every} level function \(L_P\) must have \emph{only} convex fibers. More complex decision rules may also produce non-convex regions; this is not a defect in a mathematical or governance sense. In such cases, the distribution of states remains well defined, while convex compositions and their governance evaluations (see Section~\ref{section:population_governance}) can be restricted to convex subregions whose  applicability to governance needs to be assessed case by case. In Appendix~\ref{appendix:non_convex} we illustrate an example of a function \(L_P\) with convex and non-convex fibers.

% %%%%%%%%%%%%%%%%%%%%%%%%%%%%%%%%%%%%%%%%%%%
% \subsection{Addendum: An Example of Local Congruence Under Homogeneous Transformations of AI States}
% An explicit case in which reachability monotonicity holds as in Proposition~\ref{prop:identity_compatible_composition} arises when lifecycle transformations are \emph{locally} homogeneous. Fix a set of profile dimensions \(J\subseteq [n]\) and consider an improving relation \(R_J^{+}\), defined by coordinatewise inequalities \(\pi_j(q_x)\leq\pi_j(q_y)\) for all \(j\in J\), or the corresponding degrading relation \(R_J^{-}\) obtained by reversing the inequalities. On a region in which admissible paths between states \(x,y\) contain only improving transformations, or only degrading transformations, reachability reduces locally to the corresponding coordinatewise preorder on states. Since \(\{\omega_{\alpha,k}\}_\alpha\) preserve coordinatewise inequalities, barycentric composition is monotone for this reachability relation and Proposition~\ref{prop:identity_compatible_composition} holds. Further algebraic properties of this construction are given in Appendix~\ref{appendix:homogeneous_transformations}.

%%%%%%%%%%%%%%%%%%%%%%%%%%%%%%%%%%%%%%%%%%%%%%%%%%%%%
%%%%%%%%%%%%%%%%%%%%%%%%%%%%%%%%%%%%%%%%%%%%%%%%%%%%%
\section{\StepII: Population-Level Lifecycle Governance Objects}
\label{section:population_governance}
In the previous sections, we fixed the domain within which AI instantiations can be compared, that is the type \(\diam\), and defined the convex compositions \(\{\omega_{\alpha,k}\}_\alpha\) of their states as in~\eqref{eq:convex_comp_states}. We now introduce two population-level objects for lifecycle governance: the \emph{state-population distribution}, representing how the population is distributed across those states, and the \emph{convex hull}, which represents the configuration generated by the states occupied by an AI population. They are standard objects used in statistical mechanics and convex geometry \citep{berlinsky2019statistical,hug2020lectures}. Their usefulness resides in their interpretation as governance objects, which is postponed to Section~\ref{section:discussion}.

Let \(\mathcal P=\{u_1,\ldots,u_m\}\) be a finite AI population identified under the common type \(\diam\). At time \(t\), its
\(m\) instantiations occupy \(N\leq m\) distinct states \(X_t=\{x_1(t),\ldots,x_N(t)\}\subseteq S_\diam\). (Several AI instantiations may occupy the same state.)

\begin{definition}[AI state population distribution]
\label{def:AI_state_pop_distr}
Let \((\lambda_1(t),\dots,\lambda_N(t))\in\Delta^{N-1}\) denote the governance weights associated with
the states in \(X_t\) and
\(
\Delta^{N-1}=\left\{
(\beta_1,\ldots,\beta_N)\in[0,1]^{N}\;\middle|\;\sum_{i=1}^{N}\beta_i=1\right\}\) the standard \(N-1\)-simplex. The \emph{AI state population
distribution} at time \(t\) is
\begin{align}
\mu_t(\mathcal P):=\sum_{j=1}^{N}\lambda_j(t)\delta_{x_j(t)}.
\end{align}
Its barycenter \(\bar q_{\mu_t(\mathcal P)}(t)\) and weighted state variance \(V_{\mu_t(\mathcal P)}(t)\) are
\begin{align}
\bar q_{\mu_t(\mathcal P)}(t):=\sum_{j=1}^{N}\lambda_j(t)q_{x_j(t)},\quad
V_{\mu_t(\mathcal P)}(t):=
\sum_{j=1}^{N}\lambda_j(t)
\|q_{x_j(t)}-\bar q_{\mu_t(\mathcal P)}(t)\|_2^2.
\end{align} 
\end{definition}

Here, \(\delta_x\) denotes unit mass concentrated at state \(x\)\footnote{
For a set \(A\subseteq S_\diam\), the Dirac distribution satisfies \(\delta_x(A)=1\) when \(x\in A\) and \(0\) otherwise.} and \(\|\cdot\|_2\) the Euclidean norm in \(\mathbb{R}^n.\)
The semantics of the weights \(\lambda_j(t)\) is determined by the AI governance task at hand: for instance, \(\lambda_j(t)\) may represent fractions of AI instantiations at each state in \(X_t\), as we will discuss in Section~\ref{section:discussion}. Accordingly, \(\mu_t(\mathcal P)\) records the distribution  of the selected governance weights across quantified profiles. (When
\(\lambda_j(t)\) represents the fraction of AI instantiations occupying state \(x_j(t)\), this coincides with the distribution of AI instantiations across those states.) Its barycenter \(\bar q_{\mu_t(\mathcal P)}(t)\) summarizes its \emph{central} profile and the weighted state variance \(V_{\mu_t(\mathcal P)}(t)\) summarizes its \emph{heterogeneity}. The barycenter can be written as a sequential composition of the (binary) convex combination maps in \(I_P\). 
For instance, for \(N=2\), \(\bar q_{\mu_t(\mathcal P)}(t)=\omega^{I_P}_{\lambda_1(t)}(q_{x_1(t)},q_{x_2(t)}),\) \(\lambda_2(t)=1-\lambda_1(t)\).
% \footnote{The construction extends recursively to \(N>3\).} 
% For \(N=3\), \(\bar q_{\mu_t(\mathcal P)}(t)=
% \omega_{\alpha_2(t)}(\omega_{\alpha_1(t)}(q_{x_1(t)},q_{x_2(t)}),q_{x_3(t)}),\)
% where \(\lambda_1(t)=\alpha_1(t)\alpha_2(t)\), \(\lambda_2(t)=(1-\alpha_1(t))\alpha_2(t)\), and \(\lambda_3(t)=1-\alpha_2(t)\). The construction extends recursively to \(N>3\). When all anchor states lie in a common fiber \(S_{\diam,k}\) with
% convex \(C_k\), 
% \(\bigl(\bar q_{\mu_t(\mathcal P)}(t),k\bigr)\) lies in \(S_{\diam,k}\).

\begin{definition}[AI population configuration hull]
Let
\(X_t=\{x_1(t),\ldots,x_N(t)\}\subseteq S_\diam\) be the distinct states occupied by population \(\mathcal P\) at time \(t\).
Its \emph{AI population configuration hull} is
\begin{align}
H_t(\mathcal P):
% =\operatorname{co}
% \{q_{x_1(t)},\ldots,q_{x_N(t)}\}
=
\left\{
\sum_{i=1}^{N}\beta_i q_{x_i(t)}
\;\middle|\;
(\beta_1,\dots,\beta_N)\in\Delta^{N-1}
\right\} \subseteq I_P, \label{eq:convex_hull}
\end{align}
namely, the smallest convex set containing the \(N\) occupied profiles of the states in \(X_t\).
\end{definition}

Equivalently, the convex hull \(H_t(\mathcal P)\) contains all quantified
profiles obtained by assigning non-negative weights summing to one to the
\(N\) states---called \emph{anchor states} or \emph{generators} of the
hull---occupied by the elements in \(\mathcal P\). Note that, when \(X_t\subseteq S_{\diam,k}\) and \(C_k\) is convex,
then \(H_t(\mathcal P)\subseteq C_k\). 
The coefficients \(\beta\in\Delta^{N-1}\) in
Equation~\eqref{eq:convex_hull} range over all such possible weightings.
By contrast, the population weights \(\lambda(t)\) introduced in Definition~\ref{def:AI_state_pop_distr} specify one particular allocation at time \(t\). When
\(\lambda_i(t)=m_i(t)/m\) represents the fraction of AI instantiations
occupying anchor \(x_i(t)\), the corresponding barycenter
\(\bar q_{\mu_t(\mathcal P)}(t)\) is the point of the hull obtained by
setting \(\beta=\lambda(t)\). Figure~\ref{figure:cube_zoom_c3_c5_hulls}(c) illustrates this distinction between AI population allocation and hull geometry. Four fixed anchor states remain unchanged between \(t_0\) and \(t_1\), and therefore
generate the same convex hull, while their occupancy frequencies
\(\lambda_i(t)\) of an AI population are redistributed across them. The barycenter consequently moves
within the fixed hull and the weighted variance increases from
\(1.01\times10^{-2}\) to \(1.57\times10^{-2}\), while
\(\|\lambda(t_1)-\lambda(t_0)\|_1=0.50\) summarizes the magnitude of the
redistribution. As \(\lambda_i(t)\) are relative frequencies,
the representation does not depend on the  number \(m\) of AI instantiations in the population.

%%%%%%%%%%%%%%%%%%%%%%%%%%%%%%%%%%%%%%%%%%%%%%%%%%%
% CUBE AND ZOOM ON C_3 / C_5 WITH CONVEX HULLS
%%%%%%%%%%%%%%%%%%%%%%%%%%%%%%%%%%%%%%%%%%%%%%%%%%%
\begin{figure}[h!]
\centering

% Common viewing angle
\tdplotsetmaincoords{60}{40}

\tikzset{
  paneltitle/.style={font=\small\bfseries},
  axis3d/.style={-{Stealth[length=2.3mm]}, line width=0.8pt},
  axis3dsoft/.style={-{Stealth[length=1.8mm]}, line width=0.35pt, draw=black!45},
  x2lab/.style={font=\scriptsize, text=black!55},
  cubeface/.style={draw=black, line width=0.45pt},
  levellab/.style={font=\scriptsize\bfseries},
  hullline/.style={draw=black, line width=0.45pt},
  hulldots/.style={draw=black, dotted, line width=0.25pt},
  hullpt/.style={circle, fill=black, inner sep=1.05pt},
  zoomline/.style={
    draw=black!45,
    line width=0.18pt,
    dash pattern=on 0pt off 1.15pt,
    line cap=round}
}

\makebox[\textwidth][c]{%

%%%%%%%%%%%%%%%%%%%%%%%%%%%%%%%%%%%%%%%%%%%%%%%%
% PANEL (a): FULL CUBE
%%%%%%%%%%%%%%%%%%%%%%%%%%%%%%%%%%%%%%%%%%%%%%%%
\begin{minipage}[t]{0.32\textwidth}
\vspace{0pt}
\centering
\small\bfseries (a)\par
\vspace{2em}

\begin{tikzpicture}[
  tdplot_main_coords,
  remember picture,
  scale=3.0,
  line join=round,
  line cap=round
]

% C2
\filldraw[cubeface,fill=black!18,opacity=0.75]
  (0,0.4,0) -- (0.7,0.4,0) -- (0.7,0.4,1) -- (0,0.4,1) -- cycle;
\filldraw[cubeface,fill=black!18,opacity=0.75]
  (0.7,0.4,0) -- (0.7,1,0) -- (0.7,1,1) -- (0.7,0.4,1) -- cycle;
\filldraw[cubeface,fill=black!18,opacity=0.75]
  (0,0.4,1) -- (0.7,0.4,1) -- (0.7,1,1) -- (0,1,1) -- cycle;
\node[levellab] at (0.30,0.72,0.94) {$\mathbf{2}$};

% C4
\filldraw[cubeface,fill=black!10,opacity=0.78]
  (0.7,0.4,0) -- (1,0.4,0) -- (1,0.4,0.8) -- (0.7,0.4,0.8) -- cycle;
\filldraw[cubeface,fill=black!10,opacity=0.78]
  (1,0.4,0) -- (1,1,0) -- (1,1,0.8) -- (1,0.4,0.8) -- cycle;
\filldraw[cubeface,fill=black!10,opacity=0.78]
  (0.7,0.4,0.8) -- (1,0.4,0.8) -- (1,1,0.8) -- (0.7,1,0.8) -- cycle;
\node[levellab] at (0.88,0.70,0.38) {$\mathbf{4}$};

% C6
\filldraw[cubeface,fill=black!3,opacity=0.82]
  (0.7,0.4,0.8) -- (1,0.4,0.8) -- (1,0.4,1) -- (0.7,0.4,1) -- cycle;
\filldraw[cubeface,fill=black!3,opacity=0.82]
  (1,0.4,0.8) -- (1,1,0.8) -- (1,1,1) -- (1,0.4,1) -- cycle;
\filldraw[cubeface,fill=black!3,opacity=0.82]
  (0.7,0.4,1) -- (1,0.4,1) -- (1,1,1) -- (0.7,1,1) -- cycle;
\node[levellab] at (0.88,0.70,0.99) {$\mathbf{6}$};

% C1
\filldraw[cubeface,fill=black!26,opacity=0.75]
  (0,0,0) -- (0.7,0,0) -- (0.7,0,1) -- (0,0,1) -- cycle;
\filldraw[cubeface,fill=black!26,opacity=0.75]
  (0.7,0,0) -- (0.7,0.4,0) -- (0.7,0.4,1) -- (0.7,0,1) -- cycle;
\filldraw[cubeface,fill=black!26,opacity=0.75]
  (0,0,1) -- (0.7,0,1) -- (0.7,0.4,1) -- (0,0.4,1) -- cycle;
\node[levellab] at (0.30,0.12,0.50) {$\mathbf{1}$};

% C3
\filldraw[cubeface,fill=black!12,opacity=0.80]
  (0.7,0,0) -- (1,0,0) -- (1,0,0.8) -- (0.7,0,0.8) -- cycle;
\filldraw[cubeface,fill=black!12,opacity=0.80]
  (1,0,0) -- (1,0.4,0) -- (1,0.4,0.8) -- (1,0,0.8) -- cycle;
\filldraw[cubeface,fill=black!12,opacity=0.80]
  (0.7,0,0.8) -- (1,0,0.8) -- (1,0.4,0.8) -- (0.7,0.4,0.8) -- cycle;
\node[levellab] at (0.75,0.12,0.35) {$\mathbf{3}$};

% C5
\filldraw[cubeface,fill=black!5,opacity=0.83]
  (0.7,0,0.8) -- (1,0,0.8) -- (1,0,1) -- (0.7,0,1) -- cycle;
\filldraw[cubeface,fill=black!5,opacity=0.83]
  (1,0,0.8) -- (1,0.4,0.8) -- (1,0.4,1) -- (1,0,1) -- cycle;
\filldraw[cubeface,fill=black!5,opacity=0.83]
  (0.7,0,1) -- (1,0,1) -- (1,0.4,1) -- (0.7,0.4,1) -- cycle;
\node[levellab] at (0.88,0.12,0.99) {$\mathbf{5}$};

%%%%%%%%%%%%%%%%%%%%%%%%%%%%%%%%%%%%%%%%%%%%%%%%%%%%%
% Hidden rear edges completing the cube
%%%%%%%%%%%%%%%%%%%%%%%%%%%%%%%%%%%%%%%%%%%%%%%%%%%%%

% Rear-bottom edge behind C2
\draw[
  draw=black!55,
  line width=0.30pt
]
  (0,1,0) -- (0.7,1,0);

% Rear-bottom edge behind C4
\draw[
  draw=black!55,
  line width=0.30pt
]
  (0.7,1,0) -- (1,1,0);

% Far rear-left vertical outer edge
\draw[
  draw=black!55,
  line width=0.30pt
]
  (0,1,0) -- (0,1,1);

% Axes
\draw[axis3d] (0,0,0) -- (1.12,0,0)
  node[anchor=west] {$x_1$};

\draw[axis3dsoft] (0,0,0) -- (0,1.12,0)
  node[x2lab,anchor=south] {$x_2$};

\draw[axis3d] (0,0,0) -- (0,0,1.15)
  node[anchor=east] {$x_3$};

%%%%%%%%%%%%%%%%%%%%%%%%%%%%%%%%%%%%%%%%%%%%%%%%%%%%%
% Zoom anchors: C3/C5 pillar -> panel (b)
% Right-front vertical edge of the C3/C5 pillar
%%%%%%%%%%%%%%%%%%%%%%%%%%%%%%%%%%%%%%%%%%%%%%%%%%%%%
\coordinate (zoomA1-top) at (0.7,0,1);
\coordinate (zoomA1-mid) at (0.7,0,0.8);
\coordinate (zoomA1-bot) at (0.7,0,0);

\coordinate (zoomA2-top) at (0.7,0.4,1);
\coordinate (zoomA2-mid) at (0.7,0.4,0.8);
\coordinate (zoomA2-bot) at (0.7,0.4,0);

\coordinate (zoomA3-top) at (1,0.4,1);
\coordinate (zoomA3-mid) at (1,0.4,0.8);
\coordinate (zoomA3-bot) at (1,0.4,0);

\coordinate (zoomA4-top) at (1,0,1);
\coordinate (zoomA4-mid) at (1,0,0.8);
\coordinate (zoomA4-bot) at (1,0,0);

\end{tikzpicture}
\end{minipage}%
\hspace{0.02\textwidth}%

%%%%%%%%%%%%%%%%%%%%%%%%%%%%%%%%%%%%%%%%%%%%%%%%
% PANEL (b): ZOOM ON C_3 AND C_5 + H_{t_0}, H_{t_1}
%%%%%%%%%%%%%%%%%%%%%%%%%%%%%%%%%%%%%%%%%%%%%%%%
\begin{minipage}[t]{0.30\textwidth}
\vspace{0pt}
\centering
\small\bfseries (b)\par
\vspace{2em}

\begin{tikzpicture}[
  tdplot_main_coords,
  scale=4.2,
  remember picture,
  line join=round,
  line cap=round
]

% Bottom block = C3
\filldraw[cubeface,fill=black!12,opacity=0.80]
  (0,0,0) -- (0.3,0,0) -- (0.3,0,0.8) -- (0,0,0.8) -- cycle;

\filldraw[cubeface,fill=black!12,opacity=0.80]
  (0.3,0,0) -- (0.3,0.4,0) -- (0.3,0.4,0.8) -- (0.3,0,0.8) -- cycle;

\filldraw[cubeface,fill=black!12,opacity=0.80]
  (0,0,0.8) -- (0.3,0,0.8) -- (0.3,0.4,0.8) -- (0,0.4,0.8) -- cycle;

\node[levellab,anchor=east]
  at (-0.02,0.18,0.26)
  {$\mathbf{3}$};

% Top block = C5
\filldraw[cubeface,fill=black!5,opacity=0.83]
  (0,0,0.8) -- (0.3,0,0.8) -- (0.3,0,1.0) -- (0,0,1.0) -- cycle;

\filldraw[cubeface,fill=black!5,opacity=0.83]
  (0.3,0,0.8) -- (0.3,0.4,0.8) -- (0.3,0.4,1.0) -- (0.3,0,1.0) -- cycle;

\filldraw[cubeface,fill=black!5,opacity=0.83]
  (0,0,1.0) -- (0.3,0,1.0) -- (0.3,0.4,1.0) -- (0,0.4,1.0) -- cycle;

\node[levellab,anchor=east]
  at (0.22,0.16,1.13)
  {$\mathbf{5}$};

%%%%%%%%%%%%%%%%%%%%%%%%%%%%%%%%%%%%%%%%%%%%%%%%%%%%%
% Hidden rear edges completing the C3/C5 pillar
%%%%%%%%%%%%%%%%%%%%%%%%%%%%%%%%%%%%%%%%%%%%%%%%%%%%%

% Rear-bottom edge
\draw[
  draw=black!55,
  line width=0.30pt
]
  (0,0.4,0) -- (0.3,0.4,0);

% Far rear-left vertical outer edge
\draw[
  draw=black!55,
  line width=0.30pt
]
  (0,0.4,0) -- (0,0.4,1);

% --- H_{t_2} in C5 (upper hull) ---
\coordinate (A0) at (0.05,0.07,0.985);
\coordinate (B0) at (0.26,0.06,0.988);
\coordinate (C0) at (0.22,0.28,0.992);
\coordinate (D0) at (0.08,0.24,0.987);

% --- H_{t_3} in C3 (lower hull) ---
\coordinate (A1) at (0.03,0.09,0.500);
\coordinate (B1) at (0.15,0.08,0.515);
\coordinate (C1) at (0.17,0.23,0.552);
\coordinate (D1) at (0.05,0.21,0.520);

% Hull lines H_{t_2}
\draw[hullline]
  (A0) -- (B0) -- (C0) -- (D0) -- cycle;

% Hull lines H_{t_3}
\draw[hullline]
  (A1) -- (B1) -- (C1) -- (D1) -- cycle;

% Correspondence dotted lines
\draw[hulldots] (A0) -- (A1);
\draw[hulldots] (B0) -- (B1);
\draw[hulldots] (C0) -- (C1);
\draw[hulldots] (D0) -- (D1);

% Hull points
\node[hullpt] at (A0) {};
\node[hullpt] at (B0) {};
\node[hullpt] at (C0) {};
\node[hullpt] at (D0) {};

\node[hullpt] at (A1) {};
\node[hullpt] at (B1) {};
\node[hullpt] at (C1) {};
\node[hullpt] at (D1) {};

% Barycenters as stars
\node at (0.1475,0.1625,0.988) {$\star$};
\node at (0.1300,0.1525,0.55) {$\star$};

% Hull labels
\node[font=\scriptsize]
  at (0.255,0.295,1.03)
  {$H_{t_2}$};

\node[font=\scriptsize]
  at (0.24,0.235,0.537)
  {$H_{t_3}$};

% Axes
\draw[axis3d]
  (0,0,0) -- (0.42,0,0)
  node[anchor=west] {$x_1$};

\draw[axis3dsoft]
  (0,0,0) -- (0,0.52,0)
  node[x2lab,anchor=south] {$x_2$};

\draw[axis3d]
  (0,0,0) -- (0,0,1.15)
  node[anchor=east] {$x_3$};

%%%%%%%%%%%%%%%%%%%%%%%%%%%%%%%%%%%%%%%%%%%%%%%%%%%%%
% Zoom anchors corresponding to panel (a)
%%%%%%%%%%%%%%%%%%%%%%%%%%%%%%%%%%%%%%%%%%%%%%%%%%%%%
\coordinate (zoomB1-top) at (0,0,1);
\coordinate (zoomB1-bot) at (0,0,0);

\coordinate (zoomB2-top) at (0,0.4,1);
\coordinate (zoomB2-bot) at (0,0.4,0);

\coordinate (zoomB3-top) at (0.3,0.4,1);
\coordinate (zoomB3-bot) at (0.3,0.4,0);

\coordinate (zoomB4-top) at (0.3,0,1);
\coordinate (zoomB4-bot) at (0.3,0,0);

\end{tikzpicture}
\end{minipage}%
\hspace{0.02\textwidth}%

%%%%%%%%%%%%%%%%%%%%%%%%%%%%%%%%%%%%%%%%%%%%%%%
% PANEL (c): FIXED HULL, CHANGING ALLOCATIONS
%%%%%%%%%%%%%%%%%%%%%%%%%%%%%%%%%%%%%%%%%%%%%%%
\begin{minipage}[t]{0.34\textwidth}
\vspace{0pt}
\centering
\small\bfseries (c)\par
\vspace{2em}

%%%%%%%%%%%%%%%%%%%%%%%%%%%%%%%%%%%%%%%%%%%%%%%%%%%%%
% User controls
%%%%%%%%%%%%%%%%%%%%%%%%%%%%%%%%%%%%%%%%%%%%%%%%%%%%%
\def\PanelCScale{1.00}          % overall scale of panel (c)
\def\HullScaleC{4.50}           % enlarges both hull copies uniformly

\def\TopHullShift{2.00cm}       % base vertical shift of H_{t_0}
\def\TopHullOffset{-0.75cm}     % extra offset: more negative = further down
\def\BottomHullShift{-0.8cm}   % vertical shift of H_{t_1}

\def\TimeArrowXShift{2.75cm}    % distance of time arrow from hull
\def\TimeArrowPad{0.55cm}       % extension beyond t_0/t_1 ticks

\def\MetricsYOffset{-0.85cm}    % distance of metrics strip below H_{t_1}

\begin{tikzpicture}[
    tdplot_main_coords,
    scale=\PanelCScale,
    line join=round,
    line cap=round
]

%%%%%%%%%%%%%%%%%%%%%%%%%%%%%%%%%%%%%%%%%%%%%%%%%%%%%
% H_{t_0}
%%%%%%%%%%%%%%%%%%%%%%%%%%%%%%%%%%%%%%%%%%%%%%%%%%%%%
\begin{scope}[yshift=\dimexpr\TopHullShift+\TopHullOffset\relax]

% same hull coordinates as in panel (b), uniformly enlarged
% q_1 = upper-right, then counterclockwise q_2, q_3, q_4
\coordinate (q1t0) at ({\HullScaleC*1.02},{\HullScaleC*0.28},{\HullScaleC*0.992});
\coordinate (q2t0) at ({\HullScaleC*0.88},{\HullScaleC*0.24},{\HullScaleC*0.987});
\coordinate (q3t0) at ({\HullScaleC*0.85},{\HullScaleC*0.07},{\HullScaleC*0.985});
\coordinate (q4t0) at ({\HullScaleC*1.06},{\HullScaleC*0.06},{\HullScaleC*0.988});

% barycenter position -- display coordinates only
\coordinate (g0) at ({\HullScaleC*0.9030},{\HullScaleC*0.1150},{\HullScaleC*0.98645});

% hull
\draw[hullline] (q1t0) -- (q2t0) -- (q3t0) -- (q4t0) -- cycle;

% hull points
\node[hullpt] at (q1t0) {};
\node[hullpt] at (q2t0) {};
\node[hullpt] at (q3t0) {};
\node[hullpt] at (q4t0) {};

% labels + local weights
\node[font=\tiny, align=left, anchor=west]
at ($(q1t0)+(0.03,0.04,0.01)$)
{$q_1,\ \lambda_1(t_0)=0.10$};

\node[font=\tiny, align=right, anchor=north]
at ($(q2t0)+(0,0.3,0.4)$)
{$q_2,\ \lambda_2(t_0)=0.20$};

\node[font=\tiny, align=right, anchor=east]
at ($(q3t0)+(-0.03,-0.02,0.01)$)
{$q_3,\ \lambda_3(t_0)=0.60$};

\node[font=\tiny, align=left, anchor=west]
at ($(q4t0)+(0.03,-0.02,0.01)$)
{$q_4,\ \lambda_4(t_0)=0.10$};

% barycenter
\node[font=\normalsize] at (g0) {$\star$};

\end{scope}

%%%%%%%%%%%%%%%%%%%%%%%%%%%%%%%%%%%%%%%%%%%%%%%%%%%%%
% H_{t_1}
%%%%%%%%%%%%%%%%%%%%%%%%%%%%%%%%%%%%%%%%%%%%%%%%%%%%%
\begin{scope}[yshift=\BottomHullShift]

% same hull geometry
\coordinate (q1t1) at ({\HullScaleC*1.02},{\HullScaleC*0.28},{\HullScaleC*0.992});
\coordinate (q2t1) at ({\HullScaleC*0.88},{\HullScaleC*0.24},{\HullScaleC*0.987});
\coordinate (q3t1) at ({\HullScaleC*0.85},{\HullScaleC*0.07},{\HullScaleC*0.985});
\coordinate (q4t1) at ({\HullScaleC*1.06},{\HullScaleC*0.06},{\HullScaleC*0.988});

% barycenter position -- display coordinates only
\coordinate (g1) at ({\HullScaleC*0.9245},{\HullScaleC*0.0945},{\HullScaleC*0.98645});

% hull
\draw[hullline] (q1t1) -- (q2t1) -- (q3t1) -- (q4t1) -- cycle;

% hull points
\node[hullpt] at (q1t1) {};
\node[hullpt] at (q2t1) {};
\node[hullpt] at (q3t1) {};
\node[hullpt] at (q4t1) {};

% labels + local weights
\node[font=\tiny, align=left, anchor=west]
at ($(q1t1)+(0.03,0.04,0.01)$)
{$q_1,\ \lambda_1(t_1)=0.05$};

\node[font=\tiny, align=right, anchor=north]
at ($(q2t1)+(0,0.3,0.4)$)
{$q_2,\ \lambda_2(t_1)=0.10$};

\node[font=\tiny, align=right, anchor=east]
at ($(q3t1)+(-0.03,-0.02,0.01)$)
{$q_3,\ \lambda_3(t_1)=0.50$};

\node[font=\tiny, align=left, anchor=west]
at ($(q4t1)+(0.03,-0.02,0.01)$)
{$q_4,\ \lambda_4(t_1)=0.35$};

% barycenter
\node[font=\normalsize] at (g1) {$\star$};

\end{scope}

%%%%%%%%%%%%%%%%%%%%%%%%%%%%%%%%%%%%%%%%%%%%%%%%%%%%%
% Same-anchor correspondence
%%%%%%%%%%%%%%%%%%%%%%%%%%%%%%%%%%%%%%%%%%%%%%%%%%%%%
\draw[hulldots] (q1t0) -- (q1t1);
\draw[hulldots] (q2t0) -- (q2t1);
\draw[hulldots] (q3t0) -- (q3t1);
\draw[hulldots] (q4t0) -- (q4t1);

%%%%%%%%%%%%%%%%%%%%%%%%%%%%%%%%%%%%%%%%%%%%%%%%%%%%%
% Time-flow arrow
%%%%%%%%%%%%%%%%%%%%%%%%%%%%%%%%%%%%%%%%%%%%%%%%%%%%%

% vertical centers of the two hull copies
\coordinate (hullCenter0) at ($(q1t0)!0.5!(q3t0)$);
\coordinate (hullCenter1) at ($(q1t1)!0.5!(q3t1)$);

% use right-hand hull edge rather than current bounding box,
% so labels/metrics do not push the arrow away
\coordinate (arrowXTop) at ([xshift=\TimeArrowXShift]q4t0);
\coordinate (arrowXBot) at ([xshift=\TimeArrowXShift]q4t1);

\coordinate (timeTickTop) at (arrowXTop |- hullCenter0);
\coordinate (timeTickBot) at (arrowXBot |- hullCenter1);

\draw[
  -{Latex[length=2.2mm,width=1.4mm]},
  line width=0.55pt
]
  ([yshift=\TimeArrowPad]timeTickTop)
  --
  ([yshift=-\TimeArrowPad]timeTickBot)
  node[pos=1,right=1pt] {$t$};

% ticks
\draw[line width=0.45pt]
  ([xshift=-0.10cm]timeTickTop)
  --
  ([xshift=0.10cm]timeTickTop)
  node[right=0.13cm] {$t_0$};

\draw[line width=0.45pt]
  ([xshift=-0.10cm]timeTickBot)
  --
  ([xshift=0.10cm]timeTickBot)
  node[right=0.13cm] {$t_1$};

%%%%%%%%%%%%%%%%%%%%%%%%%%%%%%%%%%%%%%%%%%%%%%%%%%%%%
% Compact metric comparison
%
% Numerical values use the UN-SCALED q_i profiles of the
% worked example, not the graphical display coordinates.
%%%%%%%%%%%%%%%%%%%%%%%%%%%%%%%%%%%%%%%%%%%%%%%%%%%%%

\coordinate (metricsAnchor) at ($(q3t1)!0.5!(q4t1)$);

\node[
  font=\tiny,
  anchor=north,
  align=center
] (metrics) at ([yshift=\MetricsYOffset]metricsAnchor)
{%
$\displaystyle
\begin{array}{c|c|c}
 & \bar q_{\mu_t(\mathcal P)}(t)
 & V_{\mu_t(\mathcal P)}(t)\\
\hline
t_0 & (0.86,\,0.30,\,0.89) & 1.01\times10^{-2}\\
t_1 & (0.82,\,0.25,\,0.87) & 1.57\times10^{-2}
\end{array}
$
};

\node[
  font=\tiny,
  anchor=north
] at ([yshift=0cm]metrics.south)
{$\|\lambda(t_1)-\lambda(t_0)\|_1=0.50$};

\end{tikzpicture}
\end{minipage}%
}

%%%%%%%%%%%%%%%%%%%%%%%%%%%%%%%%%%%%%%%%%%%%%%%%%%%%%
% Zoom connectors: panel (a) C3/C5 pillar -> panel (b)
%%%%%%%%%%%%%%%%%%%%%%%%%%%%%%%%%%%%%%%%%%%%%%%%%%%%%
\begin{tikzpicture}[remember picture,overlay]
% \draw[zoomline] (zoomA1-top) -- (zoomB1-top);
% %\draw[zoomline] (zoomA1-mid) -- (zoomB1-mid);
% \draw[zoomline] (zoomA1-bot) -- (zoomB1-bot);

% \draw[zoomline] (zoomA2-top) -- (zoomB2-top);
% %\draw[zoomline] (zoomA2-mid) -- (zoomB2-mid);
% \draw[zoomline] (zoomA2-bot) -- (zoomB2-bot);

\draw[zoomline] (zoomA3-top) -- (zoomB2-top);
%\draw[zoomline] (zoomA3-mid) -- (zoomB3-mid);
\draw[zoomline] (zoomA3-bot) -- (zoomB2-bot);

\draw[zoomline] (zoomA4-top) -- (zoomB1-top);
%\draw[zoomline] (zoomA4-mid) -- (zoomB4-mid);
\draw[zoomline] (zoomA4-bot) -- (zoomB1-bot);
\end{tikzpicture}

\caption{
(a) Partition of the cube \([0,1]^3\) into six regions \(C_1,\dots,C_6\). 
(b) Schematic migration of a population support over time: the quantified profiles of anchor states move from \(C_5\) at \(t_2\) to \(C_3\) at \(t_3\), thereby changing the convex hull from \(H_{t_2}\) to \(H_{t_3}\). The stars denote the associated barycenters.
(c) Governance example with four fixed anchor states in the same trustworthiness region. At \(t_0\) and \(t_1\), the anchors---and therefore the convex hull---remain unchanged, while the population allocation \(\lambda(t)=(\lambda_1(t),\dots,\lambda_4(t))\) shifts across them, moving the barycenter and changing heterogeneity. The weights represent allocation frequencies, \(\lambda_i(t)=m_i(t)/m\), so the AI population may contain arbitrarily many AI instantiations; only their relative allocation is shown. In the example, the weighted variance increases by approximately 55\%, while \(\|\lambda(t_1)-\lambda(t_0)\|_1=
\sum_{i=1}^{N}
\left|
\lambda_i(t_1)-\lambda_i(t_0)
\right|\) quantifies the magnitude of the redistribution.
}
\label{figure:cube_zoom_c3_c5_hulls}
\end{figure}
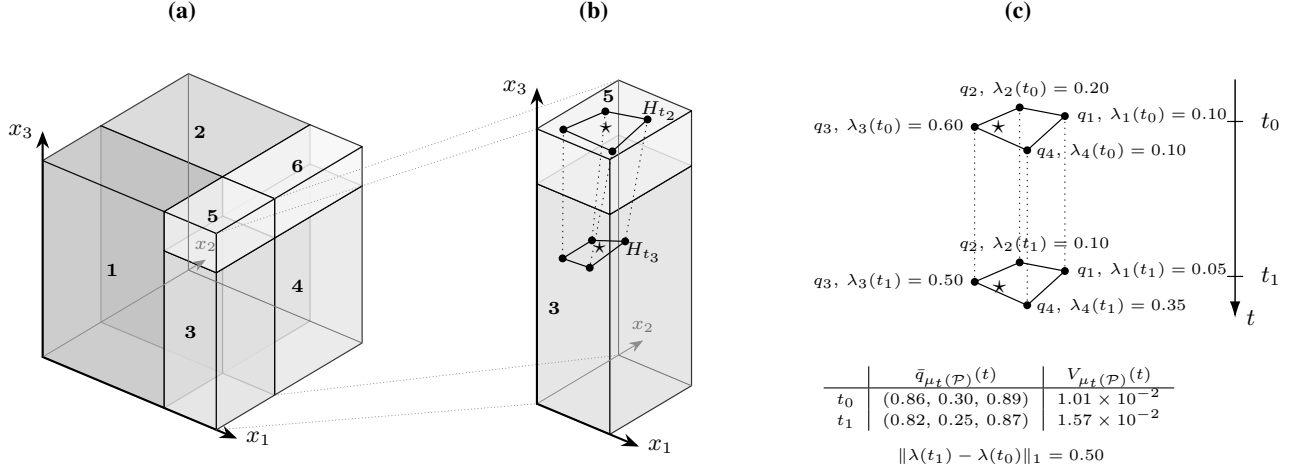
%%%%%%%%%%%%%%%%%%%%%%%%%%%%%%%%%%%%%%%%%%%%%%%%%%%

%%%%%%%%%%%%%%%%%%%%%%%%%%%%%%%%%%%%%%%%%%%%%%%%%%%%%%
%%%%%%%%%%%%%%%%%%%%%%%%%%%%%%%%%%%%%%%%%%%%%%%%%%%%%%
\section{Discussion}
\label{section:discussion}

%%%%%%%%%%%%%%%%%%%%%%%%%%%%
\subsection{AI Population Governance Workflow}
\label{subsection:AI_governance_workflow}
Figure~\ref{figure:population_governance_workflow} displays the AI population governance workflow defined by our approach. The workflow promotes the documentation and contestability of governance decisions concerning AI populations by requiring explicit choices and justification throughout \StepI~ and \StepII. Here, \StepI~ fixes the comparison domain within which AI instantiations can be treated as members of one population and provides the operationalization of the datum \(\diam\) as well as the AI identity relations required to reason about them across deployments and over time. \StepII~ then selects and designs the governance mode, defines the anchor states relevant to the governance question, and computes and monitors the mathematical objects we introduced in Section~\ref{section:population_governance}. \emph{Our use of these well-understood mathematical objects as governance objects for AI populations is deliberate: they support governance while keeping the workflow manageable}.
Now, we show how the governance role of these objects depends on the governance setup.

%%%%%%%%%%%%%%%%%%%%%%%%%%%%%%%%%
\subsection{Three Modes of AI Population Governance}
\label{subsection:governance_modes}
The AI population objects from Section~\ref{section:population_governance}  support different forms of AI population governance called \emph{governance modes}. Table~\ref{table:population_governance_strategy} displays the three governance modes, their distinct purposes, evidence, procedures, signals, and associated risks.

\emph{Retrospective governance} asks how an AI population reached a configuration that has become relevant for audit, incident analysis, or reassessment. In the healthcare example from Section~\ref{section:motivation}, this may involve reconstructing whether degradation remained confined to one hospital or appeared across several hospitals, and whether the distribution of cases contributed to the practical relevance of the problem. In the recruitment example in the same section, retrospective analysis may distinguish whether an undesirable configuration followed a change in one variant, the introduction or withdrawal of a variant, or a redistribution of inference traffic among otherwise unchanged variants. Accordingly, retrospective governance reconstructs past population distributions and hulls to trace changes in states, support, and weights, while remaining sensitive to incomplete records and selection or protocol bias.

\emph{Operational governance} concerns the AI population as it is currently deployed and evolving. Newly collected evidence may move individual instantiations to different quantified states, change the states represented in the population, or modify their relative operational weights. In the healthcare example, this supports monitoring of the states occupied by hospital deployments together with the exposure associated with them. In the recruitment example, it supports monitoring of production, canary, and alternative variants as their traffic shares change over time. Operational governance updates these objects as evidence arrives, using level crossings, new states, and weight shifts to support review, restriction, or rollback while accounting for delayed evidence and possible masking by aggregation.

%%%%%%%%%%%%%%%%%%%%%%%%%%%%%%%%%%%%%%%%%%%%%%%%%%%%%
% FIGURE: AI POPULATION GOVERNANCE WORKFLOW
%%%%%%%%%%%%%%%%%%%%%%%%%%%%%%%%%%%%%%%%%%%%%%%%%%%%%
\begin{figure*}[h!]
\centering

\begin{tikzpicture}[
    box/.style={
        % draw,
        rounded corners=2pt,
        inner sep=4pt,
        font=\scriptsize,
        minimum height=3.10cm
    },
    arr/.style={
        -{Latex[length=2mm,width=1.4mm]},
        line width=0.65pt
    },
    node distance=0.025\textwidth
]

\newcommand{\flowbullet}{%
  \raisebox{0.15ex}{\scalebox{0.55}{$\bullet$}}%
}

\newcommand{\flowcontent}[1]{%
  \parbox[c][2.05cm][c]{0.205\textwidth}{#1}%
}

\node[box, fill=stepIteal] (datum) {
\flowcontent{%
\textbf{1. Fix comparison domain \(\diam\) \citep{Ferrario2026CategoryIdentity,Ferrario2025TrustworthinessMetaphysics} for the candidate AI population}\\[2pt]
\flowbullet\,\(F\): techno-function\\
\flowbullet\,\(P\): TW profile dimensions\\
\flowbullet\,\(L_P\): TW level function
}};

\node[box, fill=stepIteal, right=of datum] (protocol) {
\flowcontent{%
\textbf{2. Operationalize \(\diam\) \citep{ferrario2026methodology}}\\[2pt]
\flowbullet\,Define metrics for the TW dimensions in the protocol \(P\)\\
\flowbullet\,Define quantified profiles \(q\in I_P\)\\
\flowbullet\,Validate TW level function \(L_P\)
}};

\node[box, fill=stepIIamber, right=of protocol] (anchors) {
\flowcontent{%
\textbf{3. Select and design governance mode, define anchor states}\\[2pt]
\flowbullet\,Define the AI governance mode and the AI population under \(\diam\)\\
\flowbullet\,Define the observed or candidate anchor states\\
\flowbullet\,Collect provenance, evidence, assumptions, and constraints for the chosen governance mode
}};

\node[box, fill=stepIIamber, right=of anchors] (population) {
\flowcontent{%
\textbf{4. Compute and monitor AI population governance objects}\\[2pt]
\flowbullet\,Population distribution and weights, convex hull, barycenter and  variance \\
\flowbullet\,Monitor their changes over time and the transformations between anchor states \\
\flowbullet\,Retain anchor-level inspection
}};

\draw[arr] (datum) -- (protocol);
\draw[arr] (protocol) -- (anchors);
\draw[arr] (anchors) -- (population);

%%%%%%%%%%%%%%%%%%%%%%%%%%%%%%%%%%%%%%%%%%%%%%%%%%%%%
% STEP I / STEP II GROUPING
%%%%%%%%%%%%%%%%%%%%%%%%%%%%%%%%%%%%%%%%%%%%%%%%%%%%%

\draw[line width=0.85pt, draw=stepIteal]
  ($(datum.north west)+(0,3mm)$)
  --
  node[
    midway,
    fill=stepIteal,
    rounded corners=2pt,
    line width=0.45pt,
    inner xsep=5.5pt,
    inner ysep=1.6pt,
    minimum height=2.35ex,
    font=\scriptsize\bfseries,
    align=center
  ] {\strut Step~I}
  ($(protocol.north east)+(0,3mm)$);

\draw[line width=0.85pt, draw=stepIIamber]
  ($(anchors.north west)+(0,3mm)$)
  --
  node[
    midway,
    fill=stepIIamber,
    rounded corners=2pt,
    line width=0.45pt,
    inner xsep=5.5pt,
    inner ysep=1.6pt,
    minimum height=2.35ex,
    font=\scriptsize\bfseries,
    align=center
  ] {\strut Step~II}
  ($(population.north east)+(0,3mm)$);

\end{tikzpicture}

\caption{
Two-step AI population governance workflow. \StepI~ fixes and operationalizes the comparison domain \(\diam\), while \StepII~ selects governance-relevant anchor states and constructs and monitors the corresponding population-level governance objects. The workflow is common to the governance modes summarized in Table~\ref{table:population_governance_strategy}. Abbreviation: TW, trustworthiness.
}
\label{figure:population_governance_workflow}
\end{figure*}
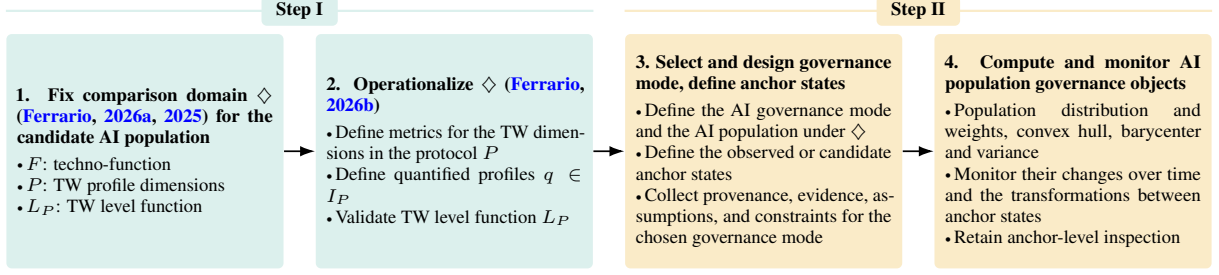

Finally, \emph{prospective governance} concerns configurations that have not yet been realized. Candidate states and weights are introduced through scenarios representing, for instance, planned updates, alternative deployment strategies, resource constraints, or different allocations of inference traffic. A healthcare provider may compare the consequences of extending an update across institutions, while a recruitment provider may evaluate alternative canary-expansion strategies before changing the live deployment. Prospective governance compares counterfactual states and weightings to identify critical configurations before deployment, subject to uncertainty from scenario and anchor specification.

%%%%%%%%%%%%%%%%%%%%%%%%%%%%%5
\subsection{Interpreting Population-Level Governance Objects across Governance Modes}
\label{subsection:interpr_gov_objects}
The interpretation and use of the population objects from Section~\ref{section:population_governance} depend on the governance mode under consideration. First, \emph{anchor state selection} is part of the AI population governance workflow---see Figure~\ref{figure:population_governance_workflow}. In retrospective and operational governance, the  anchors are the distinct states actually occupied by the AI population at the relevant time. In prospective governance, anchors may be generated through scenarios instead. They may represent, for instance, a validated historical state to be simulated, a state expected after a planned update, an adverse state that could be occupied in the future, a state attainable under specified resource constraints, or a least-favourable state still considered acceptable for deployment. The provenance, evidence, assumptions, and constraints affecting anchor states should therefore be documented explicitly, as reflected in the AI population governance workflow in Figure~\ref{figure:population_governance_workflow} and Table~\ref{table:population_governance_strategy}.

%%%%%%%%%%%%%%%%%%%%%%%%%%%%%%%%%%%%%%%%%%%%%%%%%%%%%
% TABLE: THREE MODES OF POPULATION-LEVEL GOVERNANCE
%%%%%%%%%%%%%%%%%%%%%%%%%%%%%%%%%%%%%%%%%%%%%%%%%%%%%
\begin{table}[h!]
\centering

\begin{minipage}{0.98\textwidth}
\centering
\scriptsize
\setlength{\tabcolsep}{3pt}
\renewcommand{\arraystretch}{0.95}

\newcommand{\tblbullet}[1]{%
  \par\noindent
  \hangindent=0.55em
  \hangafter=1
  \makebox[0.55em][l]{%
    \raisebox{0.15ex}{\scalebox{0.52}{$\bullet$}}%
  }#1%
}

\begin{tabularx}{\textwidth}{
    >{\raggedright\arraybackslash}p{0.090\textwidth}
    >{\raggedright\arraybackslash}p{0.135\textwidth}
    >{\raggedright\arraybackslash}p{0.190\textwidth}
    >{\raggedright\arraybackslash}p{0.190\textwidth}
    >{\raggedright\arraybackslash}p{0.195\textwidth}
    >{\raggedright\arraybackslash}X
}
\toprule

\textbf{Mode}
&
\textbf{Purpose}
&
\textbf{Evidence and anchors}
&
\textbf{Procedure}
&
\textbf{Signals and actions}
&
\textbf{Risks}
\\

\midrule

%%%%%%%%%%%%%%%%%%%%%%%%%%%%%%%%%%%%%%%%%%%%%%%%%%%%%
% RETROSPECTIVE
%%%%%%%%%%%%%%%%%%%%%%%%%%%%%%%%%%%%%%%%%%%%%%%%%%%%%
\textbf{Retrospective}
&
\tblbullet{Audit, incident analysis}
\tblbullet{Reassessment}
&
\tblbullet{Archived states and  weights}
\tblbullet{Historical anchors}
&
\tblbullet{Reconstruct \(H_t(\mathcal P),\mu_t(\mathcal P)\) over time and anchor transformations in \(\mathcal U_\diam\)}
\tblbullet{Trace state, weight, barycenter and variance changes}
&
\tblbullet{TW level transitions,  heterogeneity changes}
\tblbullet{Anchor occupancy drops to zero}
\tblbullet{Incident audit or reassessment}
&
\tblbullet{Missing or changed records}
\tblbullet{Selection or protocol bias}
\\
\hline

%%%%%%%%%%%%%%%%%%%%%%%%%%%%%%%%%%%%%%%%%%%%%%%%%%%%%
% OPERATIONAL
%%%%%%%%%%%%%%%%%%%%%%%%%%%%%%%%%%%%%%%%%%%%%%%%%%%%%
\textbf{Operational}
&
\tblbullet{Lifecycle monitoring}
\tblbullet{Post-market control}
&
\tblbullet{Current states and weights}
\tblbullet{Occupied anchors and provenance}
&
\tblbullet{Update anchors, \(H_t(\mathcal P),\mu_t(\mathcal P)\)}
\tblbullet{Track TW level and heterogeneity over time}
&
\tblbullet{Crossings, new states, weight shifts}
\tblbullet{Review, restrict, or rollback}
&
\tblbullet{Noisy or delayed evidence}
\tblbullet{Aggregate or anchor masking}
\\
\hline

%%%%%%%%%%%%%%%%%%%%%%%%%%%%%%%%%%%%%%%%%%%%%%%%%%%%%
% PROSPECTIVE
%%%%%%%%%%%%%%%%%%%%%%%%%%%%%%%%%%%%%%%%%%%%%%%%%%%%%
\textbf{Prospective}
&
\tblbullet{Pre-deployment scenarios}
\tblbullet{Population update assessment}
&
\tblbullet{Evidence, constraints, assumptions}
\tblbullet{Candidate or counterfactual anchors}
&
\tblbullet{Generate states and weights over time across scenarios}
\tblbullet{Compare resulting \(H_{t'}(\mathcal P),\mu_{t'}(\mathcal P)\) and their metrics}
&
\tblbullet{Crossings or critical configurations}
\tblbullet{Adapt or block deployment}
&
\tblbullet{Scenario misspecification}
\tblbullet{Distribution shift}
\\
\hline
\end{tabularx}

\caption{
Three modes of population-level AI governance. All modes use the workflow in
Figure~\ref{figure:population_governance_workflow}. Abbreviation: TW, trustworthiness.
% but differ in purpose,
% evidence, procedures, signals/actions, and characteristic risks.
}
\label{table:population_governance_strategy}

\end{minipage}
\end{table}
%%%%%%%%%%%%%%%%%%%%%%%%%%%%%%%%%%%%%%%%%%%%%%%%%%%%%

Depending on the governance question, \emph{population weights} \(\lambda_i(t)\)  may represent fractions of deployed instantiations, processed cases, affected individuals, inference traffic, or proportions of a controlled deployment portfolio. Different weighting conventions therefore answer different governance questions even when the anchor states remain unchanged. In the healthcare example, weighting deployments equally describes the distribution of installations, whereas weighting them by processed examinations represents case-level exposure. In the recruitment example, inference-traffic weights represent the proportion of applications processed by each active variant.
While the anchors determine which states enter the population representation, the weights  determine their governance relevance.

Finally, \emph{the state-population distribution \(\mu_t(\mathcal P)\) and the convex hull \(H_t(\mathcal P)\) play different governance roles}. The distribution \(\mu_t(\mathcal P)\) is the statistical representation of the \emph{realized} population under the selected weighting convention in the retrospective or operational governance modes and the statistical representation of a \emph{candidate} population in the prospective mode, recording how governance-relevant weights are distributed across the (observed or candidate) anchor states. Its barycenter summarizes the aggregate quantified profile associated with those weights, while its weighted variance summarizes the heterogeneity of the population around that aggregate profile. The convex hull \(H_t(\mathcal P)\) has a more abstract role. Its generators are the anchor states, but the remaining points of the hull need not correspond to states occupied by any AI instantiation. They are aggregate profiles obtained by taking convex combinations of the
anchor profiles in \(I_P\) under alternative weightings of a fixed anchor set. When all profiles lie in a common convex fiber \(C_k\), these combinations also coincide with the compositions of the maps \(\omega_{\alpha,k}\) introduced in \StepII~ and remain within \(C_k\). The governance interpretation of the hull depends on the mode. Retrospectively and operationally, it distinguishes changes in anchor support from
redistributions of AI population allocations as in Figure~\ref{figure:cube_zoom_c3_c5_hulls}.  In the healthcare example, analogous distinctions arise between changes in case exposure across otherwise unchanged hospital states and changes in the states occupied by the hospitals themselves. In the recruitment example, reallocating inference traffic among unchanged production, canary, and alternative variants changes the state-population distribution and may move its barycenter within the same hull. Introducing, withdrawing, or changing a variant modifies the anchor set and may instead modify the hull itself. In prospective governance, the convex hull defines a
\emph{counterfactual space} for comparing alternative allocations, planned updates, and stress-test scenarios for an AI population of interest. In any governance mode, a normative interpretation of the convex hull requires additional governance choices, as legal, institutional, technical, or resource constraints may restrict admissible weights to a subset of \(\Delta^{N-1}\), whose image
defines an admissible region within \(H_t(\mathcal P)\). Thus, in AI population governance, the full hull is admissible only when all weightings are permitted.

Finally, AI population governance objects should also be monitored through their \emph{changes over time}---see Figure~\ref{figure:population_governance_workflow}. For two observation times \(t<t'\), simple temporal governance metrics include changes in population allocation \(\|\lambda(t')-\lambda(t)\|_1\), barycenter displacement \(\|\bar q_{\mu_{t'}(\mathcal P)}-\bar q_{\mu_t(\mathcal P)}\|_2\), and change in heterogeneity \(|V_{\mu_{t'}(\mathcal P)}-V_{\mu_t(\mathcal P)}|\). Such deltas, and their accumulation over time, form part of the governance metric toolbox summarized in Table~\ref{table:population_governance_strategy}. For instance, consider operational governance over an interval in which the anchor set \(X=\{x_1,\ldots,x_N\}\) remains fixed, while AI instantiations migrate among these states through transformations in \(\mathcal U_\diam\) that preserve the relevant trustworthiness level as in Figure~\ref{figure:cube_zoom_c3_c5_hulls}(c). (These transformations need to be documented, as specified in Figure~\ref{figure:population_governance_workflow}.) Thus, the corresponding transitions establish \emph{reachability} between the affected anchor states, while the \(\lambda_i(t)\) may change substantially. As a result, \(\mu_t(\mathcal P)\), its barycenter, and its variance evolve, whereas the hull remains geometrically unchanged.
As a result, a change in weights can modify \(\mu_t(\mathcal P)\), its barycenter, and its variance while leaving \(H_t(\mathcal P)\) unchanged; a change in the anchor states may modify the hull instead.  Importantly, AI population governance over time does not replace anchor-level inspection. In fact, an operationally risky but low-weight state may have little influence on the barycenter. Aggregation can mask harms that are localized across deployments and over time.

%%%%%%%%%%%%%%%%%%%%%%%%%%%%%%%%%%%%%%%%%%%%%%%%%%%%%
\section{Final Recommendations and Conclusions}
\label{section:recomm_conclusion}
AI governance increasingly concerns populations of related AI system instantiations that coexist across deployments and evolve over time, while regulation remains predominantly single-system-centric. Our approach addresses this governance gap through two sequential steps covering retrospective, operational, and prospective AI governance modes by introducing \emph{meaningful} compositions of AI states and standard mathematical objects that are relevant to the governance of AI populations. The proposed workflow and governance modes recommend that governance actors (1) define and document the comparison domain before aggregating evidence across AI instantiations; (2) document anchor provenance and weighting
semantics, and distinguish changes in states, population support, and
operational exposure in relation to the selected governance mode; and (3) use population-level objects together with anchor-level inspection when deciding whether monitoring, reassessment, or corrective action is warranted. Future work should evaluate the AI population workflow in domain-specific deployments, testing whether and how its population-level representations support  institutional procedures across deployments and over time, regulatory requirements, and legal triggers.

%%%%%%%%%%%%%%%%%%%%%%%%%%%%%%%%%%%%%%%%%%%%%%%%%%%%%
%%%%%%%%%%%%%%%%%%%%%%%%%%%%%%%%%%%%%%%%%%%%%%%%%%%%%
\section*{Generative AI Usage Statement}
ChatGPT (GPT-5.6 Sol, OpenAI) was used during manuscript preparation to assist with language editing, spelling, and checks of the logical consistency of the manuscript. It was also used to generate the TikZ code for the figures, support literature searches on trustworthiness measures, draft an initial version of Table 3 in Appendix B, and check the completeness of the references. All AI-assisted outputs were subsequently reviewed, verified, and revised by the author, who takes full responsibility for the final manuscript.

%%%%%%%%%%%%%%%%%%%%%%%%%%%%%%%%%%%%%%%%%%%%%%%%%%%%%%%%%%%%%
\section*{Acknowledgments}
We acknowledge partial support by the Swiss National Science Foundation (SNSF), grant no. 229061.

%%%%%%%%%%%%%%%%%%%%%%%%%%%%%%%%%%%%%%%%%%%%%%
\clearpage
\appendix
%%%%%%%%%%%%%%%%%%%%%%%%%%%%%%%%%%%%%%%%%%%%%%%%%
\section{Step I: Summary of the Primitives and AI Identity Relations}
\label{app:primitive_comp}
Table~\ref{table:formal_primitives} collects the primitive components presented in Section~\ref{section:category_account}. The content is adapted from \citet{Ferrario2026CategoryIdentity}.

\begin{table}[ht]
\centering
\footnotesize
\renewcommand{\arraystretch}{1.22}
\begin{tabularx}{\textwidth}{
    >{\raggedright\arraybackslash}p{0.20\textwidth}
    >{\raggedright\arraybackslash}X
}
\toprule
\textbf{Symbol} & \textbf{Role in the formalization} \\
\midrule

\(F\) &
\textbf{Techno-function}: the designed technical capability that fixes the AI system kind at the selected level of abstraction. \\

\(P=\{p_1,\ldots,p_n\}\) &
\textbf{Trustworthiness profile}: the selected dimensions together with their requirements, measurement conventions, evidential conditions, normalization rules, and aggregation procedures. \\

\(q\in I_P=[0,1]^n\) &
\textbf{Quantified profile}: a particular assessment produced under \(P\). Unlike \(F,P,L_P\), it varies across states, deployments, copies, variants, and times. \\

\(L_P:I_P\to[K]\) &
\textbf{Trustworthiness level function}: the auditable rule assigning quantified profiles to a finite collection of governance-relevant levels \citep{ferrario2026methodology}. \\

\(x=(q_x,L_P(q_x))\in S_\diam\) &
\(\diam\)\textbf{-relative state}: the abstract state determined by a quantified profile and the level assigned to it. \\

\(\mathcal U_\diam\) &
\textbf{Primitive lifecycle vocabulary}: the admissible elementary transformation types considered under the fixed datum \(\diam\), each represented by a relation on \(S_\diam\). \\

\bottomrule
\end{tabularx}
\caption{Primitive components used in the present paper. Adapted from \citet{Ferrario2026CategoryIdentity}; notation and descriptions are aligned with the present construction.}
\label{table:formal_primitives}
\end{table}

%%%%%%%%%%%%%%%%%%%%%%%%%%%%%%%%%%%%%%%%%%%%%%%%%%%%%%%%%%%%
%%%%%%%%%%%%%%%%%%%%%%%%%%%%%%%%%%%%%%%%%%%%%%%%%%%%%%%%%%%%
\section{Trustworthiness Measures and Profile Encoding for Convex Composition}
\label{appendix:metric_affinity}

The quantified profile \(q\in I_P=[0,1]^n\) used in \StepII~ depends on the measurement conventions fixed by the trustworthiness protocol \(P\).
To illustrate how heterogeneous evaluations can enter this common profile space, Table~\ref{table:metric_affine_screening} synthesizes representative
trustworthiness dimensions and measures from the surveys and reviews of
\citet{kemmerzell2025towards}, \citet{kaur2022trustworthy},
\citet{li2023trustworthy}, and \citet{diaz2023connecting}. We additionally include the AI-agent reliability framework of
\citet{rabanser2026towards}, which operationalizes reliability through
consistency, robustness, predictability, and safety measures. (Here, we do not engage in an in-depth discussion on the relation between reliability and trustworthiness of AI systems. While trustworthiness includes
correct functioning in the sense captured by reliability, it is not exhausted by it. Trustworthiness expands on reliability motivating the definition, selection and operationalization of its dimensions through
broader socio-technical and normative considerations.)

For each selected measure \(m_j\), the protocol \(P\) specifies a numerical representation \(s_j:m_j\longmapsto q_j\in[0,1],\) together with its measurement, normalization, and orientation conventions. Throughout the profile, higher values of \(q_j\) are interpreted as better
with respect to the corresponding requirement. Measures for which lower values are preferable are therefore reversed after normalization, while measures with a target optimum can be represented through a documented distance-to-target score. Ordinal or categorical evaluations require an explicit numerical embedding before entering the quantified profile.
Convex composition from \StepII is performed on these quantified coordinates. Here, an interpolated coordinate through the maps \(\omega_{\alpha,k}\) in \eqref{eq:convex_comp_states} is an abstract value of a quantified profile generated through convex composition. Thus, it need not correspond to an empirically observed raw metric value, nor to the value that would result from recomputing that metric on pooled data. In summary, all these geometric constructions are relative to the numerical representation fixed by \(P\). In particular, numerical embedding alone does not imply that an ordinal or categorical measure has an intrinsic affine interpretation:
alternative admissible encodings may change convex combinations, barycenters, distances, variances, and the geometry of level regions.
Therefore, \(P\) should document the interpretation of the chosen representation and, where multiple reasonable encodings exist, the sensitivity of governance conclusions to that choice.
\begin{table}[h!]
\centering

\begin{minipage}{0.98\textwidth}
\centering
\scriptsize
\setlength{\tabcolsep}{3pt}
\renewcommand{\arraystretch}{0.95}

\newcommand{\metricbullet}[1]{%
  \par\noindent
  \hangindent=0.55em
  \hangafter=1
  \makebox[0.55em][l]{%
    \raisebox{0.15ex}{\scalebox{0.52}{$\bullet$}}%
  }#1%
}

\begin{tabularx}{\textwidth}{
    >{\raggedright\arraybackslash}p{0.185\textwidth}
    >{\raggedright\arraybackslash}p{0.475\textwidth}
    >{\raggedright\arraybackslash}X
}
\toprule

\textbf{TW dimension}
&
\textbf{Representative evaluation measures / strategies}
&
\textbf{Representation in \(P\)}
\\

\midrule

% ------------------------------------------------------------
% TECHNICAL ROBUSTNESS / SAFETY
% ------------------------------------------------------------
\textbf{Technical robustness \& safety}
&
\metricbullet{Accuracy, AUROC}
\metricbullet{Brier score, NLL, ECE}
\metricbullet{Performance under distribution shift}
\metricbullet{Certified accuracy}
\metricbullet{Attack success and fooling rates}
\metricbullet{OOD detection confidence scores}
&
Higher-is-better performance measures are normalized directly; losses,
errors, and attack rates are reversed. \(P\) fixes the test population,
shift/attack regime, and evaluation protocol.
\\

\hline

% ------------------------------------------------------------
% FAIRNESS
% ------------------------------------------------------------
\textbf{Diversity, non-discrimination \& fairness}
&
\metricbullet{Demographic/statistical parity difference}
\metricbullet{Disparate impact}
\metricbullet{Equal opportunity}
\metricbullet{Equalized odds}
\metricbullet{Predictive parity}
\metricbullet{Generalized Entropy Index}
&
Fairness measures are generally target-based. \(P\) fixes protected groups,
thresholds, targets, and the evaluation population.
\\

\hline

% ------------------------------------------------------------
% TRANSPARENCY
% ------------------------------------------------------------
\textbf{Transparency \& explainability}
&
\metricbullet{Faithfulness and fidelity}
\metricbullet{Explanation robustness}
\metricbullet{Sparsity and complexity}
\metricbullet{Information loss and gain}
\metricbullet{Human task accuracy}
\metricbullet{Task time and satisfaction}
&
Quantitative and human-evaluation measures are normalized according to
their direction. Ordinal or subjective scales require a documented
numerical encoding.
\\

\hline

% ------------------------------------------------------------
% PRIVACY
% ------------------------------------------------------------
\textbf{Privacy \& data governance}
&
\metricbullet{\(k\)-anonymity, \(l\)-diversity, \(t\)-closeness}
\metricbullet{Differential-privacy \(\epsilon,\delta\)}
\metricbullet{Membership-inference attack success rate and TPR at low FPR}
\metricbullet{Model-extraction fidelity/performance}
\metricbullet{Model-inversion performance}
&
\(P\) fixes the privacy mechanism, threat model, and admissible ranges.
Lower-is-better privacy-loss and attack measures are reversed.
\\

\hline

% ------------------------------------------------------------
% HUMAN OVERSIGHT
% ------------------------------------------------------------
\textbf{Human agency \& oversight}
&
\metricbullet{Human error rate}
\metricbullet{Task-completion time}
\metricbullet{Cognitive load}
\metricbullet{Number of user actions}
\metricbullet{Compliance rate}
\metricbullet{Trust, reliance, and trust calibration}
\metricbullet{User satisfaction}
&
Performance and usability measures are oriented consistently; reliance and
trust may require target-based or psychometric encoding.
\\

\hline

% ------------------------------------------------------------
% ACCOUNTABILITY
% ------------------------------------------------------------
\textbf{Accountability \& auditability}
&
\metricbullet{Attribution analysis}
\metricbullet{Algorithmic auditing}
\metricbullet{Legal and regulatory compliance assessment}
\metricbullet{Remediation and redress evaluation}
\metricbullet{Impact assessments}
&
These are primarily qualitative, categorical, or procedural evaluations;
few established quantitative metrics exist. Numerical inclusion in \(P\)
therefore requires an explicit operationalization.
\\

\hline

% ------------------------------------------------------------
% SOCIETAL / ENVIRONMENTAL
% ------------------------------------------------------------
\textbf{Societal \& environmental well-being}
&
\metricbullet{CO$_2$e footprint}
\metricbullet{Energy and resource consumption}
\metricbullet{Model runtime}
\metricbullet{Hardware and infrastructure efficiency}
\metricbullet{Recycling and reuse rates}
&
Physical quantities are normalized under a fixed functional unit.
Lower-is-better environmental measures are reversed.
\\

\hline

% ------------------------------------------------------------
% AGENT CONSISTENCY
% ------------------------------------------------------------
\textbf{AI-agent reliability: consistency}
&
\metricbullet{Outcome consistency}
\metricbullet{Trajectory consistency---distributional}
\metricbullet{Trajectory consistency---sequential}
\metricbullet{Resource consistency}
&
Scores are already in \([0,1]\), with higher values indicating greater
consistency. \(P\) fixes repeated-run conditions and trajectory/resource
definitions.
\\

\hline

% ------------------------------------------------------------
% AGENT ROBUSTNESS
% ------------------------------------------------------------
\textbf{AI-agent reliability: robustness}
&
\metricbullet{Fault robustness}
\metricbullet{Environment robustness}
\metricbullet{Prompt robustness}
&
Scores are already in \([0,1]\), with higher values indicating greater
robustness. \(P\) fixes the baseline and perturbation regimes.
\\

\hline

% ------------------------------------------------------------
% AGENT PREDICTABILITY
% ------------------------------------------------------------
\textbf{AI-agent reliability: predictability}
&
\metricbullet{Calibration}
\metricbullet{Confidence discrimination and AUROC}
\metricbullet{Brier-based predictability}
&
Scores are already in \([0,1]\), with higher values indicating greater
predictability. \(P\) fixes outcomes, confidence elicitation, and the
evaluation population.
\\

\hline

% ------------------------------------------------------------
% AGENT SAFETY
% ------------------------------------------------------------
\textbf{AI-agent reliability: safety}
&
\metricbullet{Constraint compliance}
\metricbullet{Harm severity conditional on violation}
&
Scores are already in \([0,1]\), with higher values indicating greater
safety. \(P\) fixes operational constraints and the harm-severity scale.
\\

\bottomrule
\end{tabularx}

\vspace{0.5em}

\parbox{0.96\textwidth}{%
\scriptsize
\textit{Encoding convention.}
Under protocol \(P\), measures are represented so that higher
\(q_j\) indicates better attainment of the corresponding requirement.
Higher-is-better measures are normalized directly, lower-is-better
measures are normalized and reversed, and target-based measures are
encoded by closeness to the specified target. Ordinal, categorical, or
qualitative evaluations require an explicitly documented numerical
encoding.
}

\caption{
Representative trustworthiness evaluation measures and strategies that may
be used to construct the quantified profile \(q\) under protocol \(P\).
The first seven rows draw primarily on the quantitative evaluation catalogue
of \citet{kemmerzell2025towards}, with broader trustworthiness framing and
additional examples from \citet{kaur2022trustworthy},
\citet{li2023trustworthy}, and \citet{diaz2023connecting}.
The final four rows summarize the AI-agent reliability dimensions and
metrics of \citet{rabanser2026towards}.
}
\label{table:metric_affine_screening}

\end{minipage}
\end{table}

%%%%%%%%%%%%%%%%%%%%%%%%%%%%%%%%%%%%%%%%%%
%%%%%%%%%%%%%%%%%%%%%%%%%%%%%%%%%%%%%%%%%%
\newpage
\section{Convex Spaces: A Universal Algebra Perspective}
\label{appendix:barycentric_basics}
In this appendix, we introduce the universal algebra background of \emph{convex spaces}, equivalently \emph{barycentric algebras} \citep{fritz2009convex}, following  \citet{smith2011modes} and \citet{zamojska2024barycentric}.
Then, we introduce an example of a convex space, namely, the free barycentric algebra over a set. Finally, we relate the free barycentric algebra to the convex hull we discussed in Section~\ref{section:population_governance}.

%%%%%%%%%%%%%%%%%%%%%%%%%%%%%%%%%%%%%%%%%%%
\subsection{Convex Spaces or Barycentric Algebras: A Universal Algebra Perspective}
In universal algebra, an algebra is a pair \((A,\Omega)\), where \(A\) is a set and \(\Omega\) is a family of finitary operations on \(A\). An operation \(\omega\in\Omega\) of arity \(r\) is a map \(\omega:A^r\longrightarrow A.\)
An algebra \((A,\Omega)\) is \emph{idempotent} when, for every \(r\)-ary operation \(\omega\in\Omega\), \(\omega(x,\ldots,x)=x.\)
Equivalently, every singleton \(\{a\}\subseteq A\) is a subalgebra \citep{smith2011modes}.

Let \(\omega,\omega'\in\Omega\) have arities \(m\) and \(n\), respectively. The algebra \(A\) is \emph{entropic} when every pair of operations satisfies the interchange identity
\begin{align}
\omega\Bigl(
    \omega'(x_{11},\ldots,x_{1n}),
    \ldots,
    \omega'(x_{m1},\ldots,x_{mn})
    \Bigr)
=
\omega'\Bigl(
    \omega(x_{11},\ldots,x_{m1}),
    \ldots,
    \omega(x_{1n},\ldots,x_{mn})
    \Bigr).
\label{eq:entropic_algebra}
\end{align}
Thus, if the elements \(x_{ij}\) are arranged as a matrix, applying the operations first row-by-row and then column-by-column yields the same result as applying them first column-by-column and then row-by-row. An algebra that is both idempotent and entropic is called a \emph{mode}. 

\begin{definition}[Convex space or barycentric algebra \citep{fritz2009convex,zamojska2024barycentric}]
\label{def:barycentric_alg}
A \emph{convex space}, equivalently a \emph{barycentric algebra}, is an
algebra \((A,\Omega)\), with
\(\Omega=\{\omega_\alpha\mid\alpha\in[0,1]\},\)
whose binary operations \(\omega_\alpha:A\times A\to A\) satisfy \(\omega_1(x,y)=x\) and, for all \(x,y,z\in A\) and \(\alpha,\beta\in[0,1]\),
\begin{align*}
\omega_\alpha(x,x)
    &=x,
    &&\text{\emph{(idempotency)}}
    \\
\omega_\alpha(x,y)
    &=\omega_{1-\alpha}(y,x),
    &&\text{\emph{(skew-commutativity)}}
    \\
\omega_\beta\!\left(\omega_\alpha(x,y),z\right)
    &=
    \omega_{\alpha\beta}\!\left(
        x,
        \omega_{\frac{\beta(1-\alpha)}
        {1-\alpha\beta}}(y,z)
    \right),
    &&\text{\emph{(skew-associativity)}},
    \quad \alpha\beta\neq1.
\end{align*}
\end{definition}
Note that, if \(\alpha\beta=1\), then \(\alpha=\beta=1,\) \(\alpha,\beta\in[0,1]\). In this case, skew-associativity is equivalent to \(\omega_1(x,\omega_\gamma(y,z))=x\) for any \(\gamma\in[0,1]\).

\begin{proposition}
Every convex space \((A,\Omega)\), with
\(\Omega=\{\omega_\alpha\mid\alpha\in[0,1]\}\), is a mode.
\end{proposition}

\begin{proof}
Idempotence of every operation \(\omega_\alpha\in\Omega\) follows from their definition.
Entropicity in the sense of
Equation~\eqref{eq:entropic_algebra} follows from skew-commutativity and skew-associativity.
\end{proof}

Definition~\ref{def:barycentric_alg} abstracts the behaviour of ordinary convex combinations. Let \(C\) be a convex subset of a real vector space \(V\). For every \(\alpha\in[0,1]\), define
\(\omega^C_\alpha(x,y):=\alpha x+(1-\alpha)y.\)
Convexity of \(C\) guarantees that \(\omega^C_\alpha(x,y)\in C\) whenever \(x,y\in C\). The family \(\Omega=\{\omega^C_\alpha\mid\alpha\in[0,1]\}\) therefore equips \(C\) with a convex space structure. This is the canonical geometric example, but a convex space is an abstract algebraic structure and does not require an ambient vector space \citep{zamojska2024barycentric}. Note that we follow the coefficient convention of \citet{fritz2009convex}, according to which \(\omega_\alpha(x,y)=\alpha x+(1-\alpha)y\) in the standard vector-space example. The notation of \citet{zamojska2024barycentric} uses \(\omega_\alpha(x,y)=(1-\alpha)x+\alpha y\) instead. Accordingly, the coefficients appearing in skew-commutativity and skew-associativity relations in Definition~\ref{def:barycentric_alg} must be transformed by the corresponding change of parameter. For convex subsets \(C\) of real vector spaces, the operations \(\omega^C_\alpha\)  additionally satisfy the \emph{cancellation property} \(
\omega^C_\alpha(x,y)=\omega^C_\alpha(x,z)\Longrightarrow y=z,\ \alpha\neq 1\) \citep{zamojska2024barycentric}.

%%%%%%%%%%%%%%%%%%%%%%%%%%%%%%%%%%%%%%%%%%%%%%%%%%%%%
\subsection{The Free Barycentric Algebra}
\label{subsection:free_barycentric_algebra}
Another example of a convex space is the \emph{free barycentric algebra} \citep{smith2011modes}. Let \(A\) be an arbitrary set. The free barycentric algebra over \(A\),
denoted by \(F_B(A)\), can be represented by formal finite convex
combinations
\begin{align*}
F_B(A)
:=
\left\{
\sum_{a\in A}\lambda_a a
\;\middle|\;
\lambda_a\in[0,1],\
\lambda_a=0\text{ for all but finitely many }a,\
\sum_{a\in A}\lambda_a=1
\right\}.
\end{align*}
Equivalently, its elements can be represented as finitely supported probability functions,
\begin{align*}
F_B(A)
\cong
\left\{
\mu:A\to[0,1]
\;\middle|\;
\supp(\mu)\text{ is finite},\
\sum_{a\in A}\mu(a)=1
\right\}.
\end{align*}

For \(a\in A\), let the Dirac function \(\delta_a:A\to[0,1]\) be defined by
\[
\delta_a(a')
=
\begin{cases}
1, & a'=a,\\
0, & a'\neq a.
\end{cases}
\]
Therefore, every element of \(F_B(A)\) can be written as
\[
\mu=\sum_{a\in A}\lambda_a\delta_a, \quad
\sum_{a\in A}\lambda_a=1,\] 
with finite support. Under this representation, the convex space operations act pointwise:
\[
\omega^{F_B(A)}_\alpha(\mu,\nu)
:=
\alpha\mu+(1-\alpha)\nu.
\]

The canonical map
\(\eta_A:A\longrightarrow F_B(A), a\longmapsto\delta_a,\)
embeds the generators into the free algebra. For a finite set \(A=\{a_1,\ldots,a_N\},\) the free barycentric algebra can be identified with the standard simplex
\begin{align*}
F_B(A)
\cong
\Delta^{N-1}
:=
\left\{
(\lambda_1,\ldots,\lambda_N)\in[0,1]^N
\;\middle|\;
\sum_{i=1}^{N}\lambda_i=1
\right\}.
\end{align*}
Thus, a point of the simplex specifies a formal distribution of weight across the \(N\) generators. This free construction is directly related to the geometric convex hull
introduced in Section~\ref{section:population_governance}. The simplex \(F_B(A)\) describes all \emph{formal} convex combinations of the generators, whereas their convex hull describes the points obtained after those formal combinations
are evaluated in the ambient convex space. The relation between the two is given by a canonical evaluation map.

%%%%%%%%%%%%%%%%%%%%%%%%%%%%%%%%%%%%%%%%%%%%%%%%%%%%%
\subsection{The Free Barycentric Algebra and the Convex Hull}
Let \(X=\{x_1,\ldots,x_N\}\subseteq S_{\diam,k}, x_i=(q_{x_i},k),\)
and let \(\lambda=(\lambda_1,\ldots,\lambda_N)\in\Delta^{N-1}.\) When the resulting quantified profile belongs to \(C_k\), define the finite convex composition
\[
\omega_{\lambda,k}(x_1,\ldots,x_N)
:=
\left(
\sum_{i=1}^{N}\lambda_iq_{x_i},
k
\right).
\]
If \(C_k\) is convex, all such compositions remain in \(S_{\diam,k}\), and the barycentric hull generated by \(X\) is
\begin{align*}
\operatorname{Hull}_k(X)
:=
\left\{
\omega_{\lambda,k}(x_1,\ldots,x_N)
\mid
\lambda\in\Delta^{N-1}
\right\}
\subseteq S_{\diam,k}.
\end{align*}

At the level of quantified profiles, this is precisely the convex hull introduced in \eqref{eq:convex_hull}, Section~\ref{section:population_governance}. The connection with the free barycentric algebra is made explicit by the following map.

\begin{proposition}
\label{prop:canonical_evaluation}
Let \(X=\{x_1,\ldots,x_N\}\subseteq S_{\diam,k}\) and assume that
\(C_k\) is convex. The map
\begin{align*}
\beta_X:F_B(X)&\longrightarrow S_{\diam,k},\quad
\sum_{i=1}^{N}\lambda_i\delta_{x_i}\mapsto
\left(\sum_{i=1}^{N}\lambda_iq_{x_i}, k \right)
\end{align*}
is a convex space homomorphism. Its image is \(\operatorname{Hull}_k(X)\).
\end{proposition}

\begin{proof}
Convexity of \(C_k\) makes \(\beta_X\) well defined. Moreover, for \(\mu,\nu\in F_B(X)\),
\[
\beta_X\!\left(\omega^{F_B(X)}_\alpha(\mu,\nu)\right)
=
\omega_{\alpha,k}\!\left(
\beta_X(\mu),
\beta_X(\nu)
\right),
\]
so \(\beta_X\) preserves the convex operations. Allowing \((\lambda_1,\ldots,\lambda_N)\) to vary over
\(\Delta^{N-1}\) produces \(\operatorname{Hull}_k(X)\).
\end{proof}

%%%%%%%%%%%%%%%%%%%%%%%%%%%%%%%%%%%%%%%%%%
%%%%%%%%%%%%%%%%%%%%%%%%%%%%%%%%%%%%%%%%%%
\section{An Example of Trustworthiness Level Function with Convex and Non-Convex Preimages}
\label{appendix:non_convex}
We provide a simple two-dimensional example illustrating a case in which two fibers \(C_k=L_P^{-1}(k)\) of a trustworthiness level function \(L_P\) are non-convex. As a result, convex combinations of profiles lying in the same fiber need not remain in that fiber. Consider the trustworthiness level function
\(L_P:[0,1]^2\to[3]:=\{1,2,3\}\)
defined by
\[
L_P(x_1,x_2)=
\begin{cases}
1, & \text{if } x_1<0.6 \text{ or } x_2<0.6,\\
2, & \text{if } x_1\geq 0.6,\ x_2\geq 0.6,\ \text{and }(x_1<0.75 \text{ or } x_2<0.75),\\
3, & \text{if } x_1\geq 0.75 \text{ and } x_2\geq 0.75.
\end{cases}
\]
Its fibers are 
\begin{align*}
&C_1=\{(x_1,x_2)\in[0,1]^2 : x_1<0.6 \text{ or } x_2<0.6\},\\
&C_2=\{(x_1,x_2)\in[0,1]^2 : x_1\geq 0.6,\ x_2\geq 0.6,\ \text{and } (x_1<0.75 \text{ or } x_2<0.75)\},\\
&C_3=\{(x_1,x_2)\in[0,1]^2 : x_1\geq 0.75,\ x_2\geq 0.75\}.
\end{align*}

\(C_2\) is non-convex: it is an L-shaped region obtained by removing the upper-right square \(C_3\) from the square \([0.6,1]\times[0.6,1]\), as displayed in Figure~\ref{figure:convex_nonconvex_levels}. Hence two points in \(C_2\) may have a convex combination lying in \(C_3\). For instance,
\((0.9,0.7)\in C_2,\) and \((0.7,0.9)\in C_2,\) but
\(\frac12(0.9,0.7)+\frac12(0.7,0.9)=(0.8,0.8)\in C_3.\)

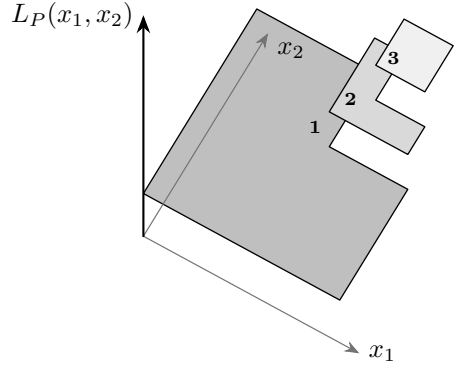
\begin{figure}[h!]
\centering
\tdplotsetmaincoords{20}{30}
\newcommand{\PanelScale}{3.00}

\tikzset{
  lab/.style={font=\scriptsize},
  axisxy/.style={-{Stealth[length=2.0mm]}, line width=0.45pt, draw=black!55},
  axisz/.style={-{Stealth[length=2.4mm]}, line width=0.8pt, draw=black},
  plateau/.style={draw=black, line width=0.5pt}
}

% Schematic heights corresponding to levels 1,2,3
\def\zOne{0.55}
\def\zTwo{1.00}
\def\zThr{1.45}
\def\zAxisTop{2.85}

\begin{tikzpicture}[
  tdplot_main_coords,
  scale=\PanelScale,
  line join=round,
  line cap=round
]

% C1: non-convex L-shaped region
\filldraw[plateau,fill=black!25]
  (0,0,\zOne) --
  (1,0,\zOne) --
  (1,0.6,\zOne) --
  (0.6,0.6,\zOne) --
  (0.6,1,\zOne) --
  (0,1,\zOne) -- cycle;

% C2: non-convex L-shaped region
\filldraw[plateau,fill=black!15]
  (0.6,0.6,\zTwo) --
  (1,0.6,\zTwo) --
  (1,0.75,\zTwo) --
  (0.75,0.75,\zTwo) --
  (0.75,1,\zTwo) --
  (0.6,1,\zTwo) -- cycle;

% C3: upper-right rectangle
\filldraw[plateau,fill=black!6]
  (0.75,0.75,\zThr) --
  (1,0.75,\zThr) --
  (1,1,\zThr) --
  (0.75,1,\zThr) -- cycle;

% Region labels
\node[lab] at (0.50,0.65,\zOne) {$\mathbf{1}$};
\node[lab] at (0.65,0.70,\zTwo) {$\mathbf{2}$};
\node[lab] at (0.80,0.82,\zThr) {$\mathbf{3}$};

% Axes
\draw[axisxy] (0,0,0) -- (1.10,0,0)
  node[anchor=west] {$x_1$};

\draw[axisxy] (0,0,0) -- (0,1.10,0)
  node[anchor=north west] {$x_2$};

\draw[axisz] (0,0,0) -- (0,0,\zAxisTop)
  node[anchor=east] {$L_P(x_1,x_2)$};

\end{tikzpicture}

\vspace{-0.8em}

\caption{
A trustworthiness level function \(L_P:[0,1]^2\to[3]\) with non-convex fibers. Both \(C_1\) and \(C_2\) are L-shaped and therefore non-convex. In particular, \(C_2\) is not closed under convex combinations: two points in \(C_2\) may have a midpoint in \(C_3\).
}
\label{figure:convex_nonconvex_levels}
\end{figure}

%%%%%%%%%%%%%%%%%%%%%%%%%%%%%%%%%%%%%%%%%%%%%%
%%%%%%%%%%%%%%%%%%
\section{An Example of Non-Trivial Reachability}
\label{appendix:reachability_example}
Consider a two-dimensional quantified profile space \(I_P=[0,1]^2,\) and suppose that the trustworthiness level function assigns the same level \(k\) throughout
\(C_k=[0.3,0.8]^2,\) for some \(k\in[K].\) Hence all profiles considered below correspond to states in the same trustworthiness fiber \(S_{\diam,k}\).
Let \(x=((0.7,0.4),k),\ y=((0.4,0.7),k).\)
We first define two primitive transformation relations. (To simplify notation, we focus on the quantified profiles of states in \(S_\diam\).) An improving transformation \(R^+\) permits a sufficiently small coordinate-wise
non-decreasing change,
\[
q\,R^+\,q'
\quad\Longleftrightarrow\quad
q'_i\geq q_i\ \text{for }i=1,2,
\qquad
\|q'-q\|_\infty\leq 0.05,
\]
while a degrading transformation \(R^-\) permits an analogous coordinate-wise non-increasing change,
\[
q\,R^-\,q'
\quad\Longleftrightarrow\quad
q'_i\leq q_i\ \text{for }i=1,2,
\qquad
\|q'-q\|_\infty\leq 0.05.
\]
Here, \(\|\cdot\|_\infty\) denotes the sup norm on \(\mathbb{R}^2\), \(\|q'-q\|_\infty=\max_{i=1,2}|q'_i-q_i|.\)
The inequalities in the definitions of \(R^{+}\) and \(R^{-}\) encode institutional constraints, for instance resource-related constraints on profile improvements or accepted degradations. If \(\mathcal U_\diam\) contains primitive transformation types represented
by \(R^+\) and \(R^-\), then the local step bound alone does not prevent reachability between distant profiles. For instance, \(x\) can first undergo a sequence of \(R^-\)-transformations,
\(
q_x=(0.7,0.4)
\to
(0.65,0.4)
\to\cdots\to
(0.4,0.4),
\)
followed by a sequence of \(R^+\)-transformations,
\(
(0.4,0.4)
\to
(0.4,0.45)
\to\cdots\to
(0.4,0.7)=q_y.
\)
As a result, every intermediate profile remains in \(C_k\): hence this path establishes \(x\preceq_\tau y\). The reverse path analogously establishes \(y\preceq_\tau x\). 

Now suppose that the admissible lifecycle transformations are additionally subject to a feasibility constraint. Let
\[
A=[0.35,0.75]\times[0.35,0.50],
\qquad
B=[0.35,0.50]\times[0.60,0.75],
\]
and define the feasible region \(F=A\sqcup B.\)
The sets \(A\) and \(B\) may represent, for instance, states attainable under
two technically or institutionally admissible operating regimes. Profiles between them remain well-defined elements of \(C_k\), but transitions through them are not admissible under the lifecycle model. In general, providers routinely distinguish among operating regimes supported by different validated workflows, resource envelopes, deployment environments, model versions, or human-oversight arrangements. For instance, a medical AI system may be validated under a standard radiology workflow and under an enhanced specialist-review workflow, while a recruitment system may be operated under distinct production and higher-oversight regimes. Both regimes may satisfy the requirements associated with the same trustworthiness level \(k\), even though intermediate combinations of the quantified dimensions have not been validated, are technically unavailable, or are institutionally inadmissible. Thus, \(A\) and \(B\) are provider- or institution-defined regions within \(C_k\), while \(F=A\sqcup B\) represents the profiles attainable under the lifecycle regimes currently recognized as feasible and admissible.

If we modify \(R^+\) and \(R^-\) by requiring additionally that the corresponding quantified profiles belong to \(F\), then two states
\(x,y\in S_{\diam,k}\) with \(q_x\in A\) and \(q_y\in B\) still have the same trustworthiness level, and both improving and degrading primitive
transformations remain available locally. Nevertheless, \(x\not\preceq_\tau y\). In fact, any path from \(A\) to \(B\) would have to cross the gap
\(0.50<q_2<0.60\). 
No profile in this gap belongs to \(F\), while any direct transition from \(A\) to \(B\) changes the second coordinate by at least
\(0.10\), violating the step bound \(0.05\). Consequently, no finite composition of admissible primitive transformations connects \(x\) to \(y\). 

This example shows that trustworthiness-preserving reachability depends \emph{on the paths generated by the primitive transformation vocabulary \(\mathcal U_\diam\) and on the technical, resource, institutional, or procedural constraints encoded by its relations}. It also distinguishes two levels of operational change. Within a fiber \(C_k\), providers or institutions may distinguish different feasible operating regimes, such as \(A\) and \(B\), and lifecycle transformations may or may not permit transitions between them. Such regime changes leave the global trustworthiness level unchanged. By contrast, crossing from \(C_k\) to \(C_{k'}\), \(k\neq k'\), changes the trustworthiness level itself and is therefore identity-relevant in the AI identity constructions \citep{Ferrario2025TrustworthinessMetaphysics,Ferrario2026CategoryIdentity}. Thus, \emph{operational change} is more fine-grained than \emph{trustworthiness level change}: substantial differences may arise within one level, while transitions across levels mark a stronger boundary. Such transitions may provide evidence
relevant to regulatory reassessment, including whether a change qualifies
as a \emph{substantial modification} under the AIA \citep{EUAIAct2024}.

%%%%%%%%%%%%%%%%%%%%%%%%%%%%%%%%%%%%%%%%%%%%%%%%%%%%%%%%%%%%%%%%%%%%%
\section{Homogeneous Transformations and Local Ordered Structure}
\label{appendix:homogeneous_transformations}
We provide a \emph{local} example of the
reachability-monotonicity condition used in
Proposition~\ref{prop:identity_compatible_composition}. The construction applies
locally to regions in which admissible lifecycle transformations are homogeneous, namely, all improving or all degrading with respect to a fixed collection of profile dimensions. This provides a simple-to-check compatibility condition with respect to the convex compositions \(\omega_{\alpha,k}(x,y)=\left(
\alpha q_x+(1-\alpha)q_y,k\right),\ \alpha\in[0,1]\) on \(S_{\diam,k}\)---see~\eqref{eq:convex_comp_states}. Fix \(k\in[K]\), a set of dimensions \(J\subseteq\{1,\ldots,n\}\), and
a region \(U\subseteq S_{\diam,k}\) that is closed under the compositions \(\omega_{\alpha,k}\). Let
\(\pi_J:I_P\longrightarrow[0,1]^{|J|}\)
denote projection onto the coordinates indexed by the set \(J\), and let
\(\pi_j:I_P\longrightarrow[0,1]\) denote projection onto coordinate
\(j\). For \(x,y\in U\), define two relations:
\[
(x,y)\in R_{J}^{+}
\quad\Longleftrightarrow\quad
\pi_j(q_x)\leq \pi_j(q_y)
\quad\text{for every }j\in J,
\]
and
\[
(x,y)\in R_{J}^{-}
\quad\Longleftrightarrow\quad
\pi_j(q_x)\geq \pi_j(q_y)
\quad\text{for every }j\in J.
\]
The relation \(R_J^{+}\) represents transformations that are weakly improving in all selected dimensions, whereas \(R_J^{-}\) represents weakly degrading transformations. A lifecycle path is \emph{\(J\)-homogeneous} when all its primitive steps belong to
\(R_J^{+}\), or all belong to \(R_J^{-}\). Let \(\preceq_{J}^{+}\) and \(\preceq_{J}^{-}\) denote reachability
generated by such homogeneous paths inside \(U\). Because the relations above are themselves reflexive and transitive,
\[
x\preceq_{J}^{+}y
\quad\Longleftrightarrow\quad
\pi_J(q_x)\leq \pi_J(q_y)
\]
coordinatewise, while
\[
x\preceq_{J}^{-}y
\quad\Longleftrightarrow\quad
\pi_J(q_x)\geq \pi_J(q_y).
\]
Indeed, every homogeneous path yields the corresponding chain of
coordinatewise inequalities. Conversely, whenever the endpoint inequality holds, the pair already belongs to \(R_J^{+}\), respectively \(R_J^{-}\), and therefore gives a one-step homogeneous path. In the local regime considered here, we assume that the restriction of lifecycle reachability \(\preceq_\tau\) to \(U\) is generated by one of these homogeneous relations. Thus, for the relevant sign \(\sigma\in\{+,-\}\),
\(\preceq_\tau\!\mid_U=\preceq_J^\sigma.\) We arrive at: 

\begin{proposition}
\label{prop:local_order_compat}
For either sign \(\sigma\in\{+,-\}\), if
\[
x_1\preceq_J^\sigma y_1,
\qquad
x_2\preceq_J^\sigma y_2,
\]
then
\[
\omega_{\alpha,k}(x_1,x_2)
\preceq_J^\sigma
\omega_{\alpha,k}(y_1,y_2)
\]
for every \(\alpha\in[0,1]\).
\end{proposition}

\begin{proof}
Consider first the improving relation. For every \(j\in J\),
\(\pi_j(q_{x_1})\leq\pi_j(q_{y_1}),\ \pi_j(q_{x_2})\leq\pi_j(q_{y_2}).
\)
Since \(\alpha,1-\alpha\geq0\),
\(
\alpha\pi_j(q_{x_1})+(1-\alpha)\pi_j(q_{x_2})
\leq
\alpha\pi_j(q_{y_1})+(1-\alpha)\pi_j(q_{y_2}).
\)
Hence
\(
\omega_{\alpha,k}(x_1,x_2)
\preceq_J^+
\omega_{\alpha,k}(y_1,y_2).
\)
The degrading case follows identically with all inequalities reversed.
\end{proof}

This gives a concrete local sufficient condition for the reachability-monotonicity assumption of
Proposition~\ref{prop:identity_compatible_composition}: on \(U\), the condition holds whenever lifecycle reachability is generated by one of the
homogeneous relations above. If \(U=S_{\diam,k}\), the same argument verifies
reachability monotonicity on the entire fiber.
In particular, define mutual \emph{homogeneous} reachability by
\(x\equiv_J y\Longleftrightarrow
x\preceq_J^+y\) and \(y\preceq_J^+x.\)
The same equivalence relation is obtained from \(\preceq_J^{-}\), and
\(x\equiv_J y\Longleftrightarrow 
\pi_J(q_x)=\pi_J(q_y).\) Then, we arrive at

\begin{corollary}
The equivalence relation \(\equiv_J\) is a congruence for the barycentric
operations. That is,
\[
x_1\equiv_J y_1,
\qquad
x_2\equiv_J y_2
\]
implies
\[
\omega_{\alpha,k}(x_1,x_2)
\equiv_J
\omega_{\alpha,k}(y_1,y_2)
\]
for every \(\alpha\in[0,1]\).
\end{corollary}
\begin{proof}
Apply Proposition~\ref{prop:local_order_compat} in both directions.
\end{proof}

\bibliographystyle{ACM-Reference-Format}
\bibliography{sample-base}

\end{document}